\documentclass{article}
\usepackage{ijcai25}

\usepackage{times}
\usepackage{soul}
\usepackage{url}
\usepackage[hidelinks]{hyperref}
\usepackage[utf8]{inputenc}
\usepackage[small]{caption}
\usepackage{graphicx}
\usepackage{amsmath}
\usepackage{amsthm}
\usepackage{mathtools}
\usepackage{booktabs}
\usepackage{algorithm}
\usepackage{algorithmic}
\usepackage[switch]{lineno}
\usepackage{subfigure}
\usepackage{multirow}
\usepackage{diagbox}
\usepackage{natbib}

\usepackage{amsmath,amsfonts,bm}

\newcommand{\R}{\mathbb R}

\newcommand{\hL}{\mathcal L}
\newcommand{\hR}{\mathcal R}

\newcommand{\mg}{\mathbf g}
\newcommand{\mmu}{\mathbf u}

\newcommand{\mx}{\mathbf x}

\newcommand{\mA}{\mathbf A}
\newcommand{\mB}{\mathbf B}
\newcommand{\mE}{\mathbf E}

\newcommand{\mC}{\mathbf C}
\newcommand{\mI}{\mathbf I}

\newcommand{\mN}{\mathcal{N}}
\newcommand{\mS}{\mathbf{S}}
\newcommand{\mU}{\mathbf U}
\newcommand{\mV}{\mathbf V}
\newcommand{\mX}{\mathbf X}
\newcommand{\mY}{\mathbf Y}

\newcommand{\mtheta}{\bm{\theta}}
\newcommand{\mSig}{\bm{\Sigma}}

\def\ie{\textit{i.e.}}
\def\eg{\textit{e.g.}}

\def\defineas{\dot{=}}

\DeclarePairedDelimiter{\@norm}{\lVert}{\rVert}
\def\norm{\@norm}

\newtheorem{theorem}{Theorem}
\newtheorem{corollary}{Corollary}
\newtheorem{lemma}{Lemma}
\newtheorem{proposition}{Proposition}

\title{Efficient Differentiable Approximation of Generalized Low-rank Regularization}

\author{
Naiqi Li$^1$
\and
Yuqiu Xie$^1$\and
Peiyuan Liu$^1$\and
Tao Dai$^{2,}$\thanks{Corresponding authors: Tao Dai and Yong Jiang.}\and
Yong Jiang$^{1,}$\footnotemark[1]\And
Shu-Tao Xia$^{1}$\\
\affiliations
$^1$Tsinghua Shenzhen International Graduate School\\
$^2$Shenzhen University\\
\emails
\{linaiqi, jiangy, xiast\}@sz.tsinghua.edu.cn,
\{jieyq22, lpy23\}@mails.tsinghua.edu.cn,
daitao.edu@gmail.com
}

\begin{document}

\maketitle

\begin{abstract}
Low-rank regularization (LRR) has been widely applied in various machine learning tasks, but the associated optimization is challenging. Directly optimizing the rank function under constraints is NP-hard in general. To overcome this difficulty, various relaxations of the rank function were studied. However, optimization of these relaxed LRRs typically depends on singular value decomposition, which is a time-consuming and nondifferentiable operator that cannot be optimized with gradient-based techniques.
To address these challenges, in this paper we propose an efficient differentiable approximation of the generalized LRR. 
The considered LRR form subsumes many popular choices like the nuclear norm, the Schatten-$p$ norm, and various nonconvex relaxations.
Our method enables LRR terms to be appended to loss functions in a plug-and-play fashion, and the GPU-friendly operations enable efficient and convenient implementation. 
Furthermore, convergence analysis is presented, which rigorously shows that both the bias and the variance of our rank estimator rapidly reduce with increased sample size and iteration steps.
In the experimental study, the proposed method is applied to various tasks, which demonstrates its versatility and efficiency. Code is available at \url{https://github.com/naiqili/EDLRR}.
\end{abstract}

\section{Introduction}

Low-rank structures have proven to be effective in a wide range of machine learning tasks, encompassing computer vision \cite{ren2022robust}, model compression \cite{idelbayev2020low}, representation learning \cite{fu2021hierarchical,fu2022one}, and large language model adaptation \cite{HuSWALWWC22}.
To discover and utilize the low-rank structures, a typical paradigm is to introduce low-rank regularization (LRR) terms to the models, which can conveniently express various low-rank priors and assumptions.

However, with the presence of LRR terms, the optimization problem can be extremely difficult. It is well-known that even under linear constraints, optimizing the rank function is NP-hard \cite{wright2022high}.
To alleviate this computational intractability, many relaxations of the rank function have been proposed.
For example, the nuclear norm is known to be the tightest convex approximation of the rank function, and thus has been extensively investigated and applied \cite{yang2016nuclear}.
The Schatten-$p$ norm and its variants are considered as the generalization of the nuclear norm, and also found successful applications \cite{xu2017unified,chen2021efficient}.
These relaxation techniques penalize all singular values of the target matrix simultaneously, while in practical applications it is desired that the large singular values are less impacted, so the key information of the target matrix is preserved.
With this motivation, recent studies proposed novel nonconvex relaxations \cite{kang2015robust,peng2015subspace,friedman2012fast,gao2011feasible}.


However, it is still challenging to efficiently optimize these models with relaxed low-rank regularization.
Current optimization methods can be categorized as the matrix factorization approach and the rank function optimization approach, but both of them suffer severe shortcomings in practice.

The basic idea of matrix factorization is to decompose the target matrix into the product of multiple low-rank sub-matrices.
This formulation allows gradient to be propagated, and so the optimization is straightforward. 
A key problem of this approach is that it demands strong prior knowledge about the matrix's true rank or its upper bound. In almost all practical situations such knowledge is unavailable, and so the good performance frequently depends on laborious parameter selection.
Arguably, the goal of rank regularization is to discover the true rank of a matrix, but this approach leaves this burden to the users.

In contrast, the rank function optimization approach incorporates a regularization term into the loss function, allowing the appropriate matrix rank to be automatically determined. However, solving the associated optimization problem is nontrivial. Previous works attempted to adapt techniques from the convex optimization literature, such as the proximal gradient and alternating direction method of multipliers (ADMM). Unfortunately, many of these methods require the loss function to be convex, which severely limits their applicability. Additionally, these approaches are inconvenient to implement, and each work needs to laboriously derive the specific optimization rules.
Critically, all these methods rely on singular value decomposition (SVD).
Besides being time-consuming, SVD is generally a nondifferentiable operation that hinders gradient propagation 
(although there exists sophisticated method for computing SVD that allows gradient propagation, it incurs $O(D^4)$ computational complexity which is prohibitive for practical applications \cite{papadopoulo2000estimating}).
Consequently, applying gradient-based optimization techniques or utilizing popular deep learning libraries (\eg, PyTorch and TensorFlow) becomes infeasible, and high-performing GPUs cannot be fully utilized.

In this paper, we propose a novel differentiable approximation of a general form of LRR, which covers a broad range of relaxations of the rank function. 
The main idea of our work is to introduce an equivalent stochastic definition of the rank function, as well as its relaxed variants.
This significant discovery enables us to approximate the LRR term with finite samples and a partial sum of the series expansion.
Our implementation is publicly available in the supplementary material. The main contributions of this paper are summarized as follows:
\begin{enumerate}
    \item We propose an efficient differentiable approximation of the generalized LRR, which covers the nuclear norm, the Schatten-$p$ norm, and many recently proposed nonconvex relaxations of the rank function.
    \item The proposed LRR approximation is convenient and efficient. It can be directly appended to loss functions and optimized by existing deep learning libraries. In most cases, a few lines of code are sufficient to adapt to new problems. The operations are GPU-friendly, which enable highly parallel and efficient computation.
    \item Theoretical convergence analysis is presented, which rigorously demonstrates that both the bias and the variance of the proposed rank estimator rapidly reduce with increased sample size and iteration steps.
    \item We performed extensive experiments over various tasks, including matrix completion, video fore-background separation, and image denoising, which demonstrate the versatility and efficiency of our proposed method.
\end{enumerate}

\section{Related Work}

\subsection{Relaxations of the Rank Function}


Minimizing the rank function is challenging, and even finding the optimal solution under the linear constraint is NP-hard \cite{wright2022high}.
To address this challenge, various relaxations of the rank function have been proposed.
The nuclear norm $\norm{\mS}_*=\sum_{i=1}^r \sigma_i(\mS)$ is the tightest convex relaxation of the rank function, and is one of the most popular substitution \cite{yang2016nuclear}.
However, the nuclear norm penalizes all singular values simultaneously. Since for many matrices in practical applications, the major information is captured by a few singular values, so it is desired that they are less impacted when reducing the rank.
Motivated by this, advanced relaxations of the general form $\hR(\mS)=\sum_{i=1}^r h(\sigma_i(\mS))$ are considered, where $h$ is a function that increases penalties on small singular values.
Typical examples of the relaxed rank function include the $\gamma$-nuclear norm \cite{kang2015robust}, Laplace \cite{trzasko2008highly}, LNN \cite{peng2015subspace}, Logarithm \cite{friedman2012fast}, ETP \cite{gao2011feasible}, and Geman \cite{geman1995nonlinear}.
A summary of these relaxations and the corresponding penalty function $h$ can be found in \cite{hu2021low}.

\subsection{Optimization of the Low-rank Regularization}



Matrix factorization and rank function optimization are two prominent approaches for optimizing low-rank models.

The concept of matrix factorization involves decomposing the target matrix into the product of multiple low-rank components. Under this formulation, the optimization is straightforward. Notable examples include the decomposition of weight matrices into low-rank factors. 
Additionally, \citeauthor{GengGCLWL21} demonstrated that this strategy enables neural networks to learn global information and can even replace the attention mechanism \cite{GengGCLWL21}. \citeauthor{ornhag2020bilinear} (\citeyear{ornhag2020bilinear}) introduced a differentiable bilinear parameterization of the nuclear norm, relying on matrix decomposition. 
The authors further proposed VarPro, which utilizes second-order optimization that enjoys faster convergence \cite{ornhag2021bilinear}. 
This idea was further extended in \cite{xu2017unified,chen2021efficient}, where the multi-Schatten-$p$ norm was considered as a generalization of the bilinear parameterization. However, a significant challenge of this approach lies in the requirement of strong prior knowledge about the true rank of the matrix.

On the other hand, the rank function optimization approach introduces a regularization term to the loss, allowing the appropriate matrix rank to be automatically determined during training. Initial research attempts applied existing convex optimization techniques to the problem, such as the proximal gradient algorithm \cite{yao2018large} and the iteratively re-weighted algorithm \cite{mohan2012iterative}.
However, these methods require the loss function to be convex.
Alternating direction method of multipliers (ADMM)  is a popular method for optimizing the regularized loss. 
\citeauthor{shang2017bilinear} (\citeyear{shang2017bilinear}) proposed the double nuclear penalty, which covers the Schatten-$p$ norm with $p=1/2$ and $2/3$. 
Recently, it is discovered that one of the optimization sub-procedures within ADMM can be seen as performing image denoising, enabling the integration of existing denoising neural networks into the ADMM framework \cite{HuJZ22,liu2023combining}. 
However, all these ADMM based methods do not support gradient propagation, so deep learning libraries and high-performing GPUs cannot be fully utilized.

There are a few works that consider SVD-free optimization of LRRs, which draw inspirations from the variational characterization of the nuclear norm, \ie, $\norm{\mX}_*=\min_{\mA \mB=\mX}\frac{1}{2}(\norm{\mA}_F^2+\norm{\mB}_F^2)$, where $\mX\in\R^{m\times n}$, $\mA\in\R^{m\times d}$, $\mB\in\R^{d\times n}$ and $d\geq rank(\mX)$ \cite{srebro2004maximum,rennie2005fast}.
Recent research further extends these methods to handle the Schatten-$p$ quasi-norm, \ie, $\norm{\mx}_p=(\sum_i |x_i|^p)^\frac{1}{p}$ with $p\in (0,1)$ \cite{shang2016scalable,fan2019factor,giampouras2020novel}.
\citeauthor{jia2020generalized} (\citeyear{jia2020generalized}) proposed GUIG for low-rank matrix recovery, and the associated bilinear variational problem can be solved without computing SVD.

To summarize, all these methods suffer at least one of the following drawbacks:
1) Their optimization is based on ADMM or its variant, which hinder gradient propagation.
2) They require the upper bound of the true rank $d$ as input. When $d$ is too large, the slow convergence brings huge computational burden, while a too small $d$ deteriorates the performance.
3) They can only optimize a restricted class of LRRs, particularly the nuclear norm and the Schatten-$p$ quasi-norm with $p\in (0,1)$. They cannot be applied to various recently proposed nonconvex LRRs like the $\gamma$-Nuclear norm \cite{kang2015robust} and Laplace \cite{trzasko2008highly}. 

\section{Methodology}

\subsection{Notations and Problem Statement}

\textbf{Notations} \quad In this paper, uppercase bold letters (\eg, $\mX$) denote matrices, and lowercase bold letters (\eg, $\mx$) denote vectors.
For a vector $\mx$, $\norm{\mx}_p=(\sum_i |x_i|^p)^\frac{1}{p}$ represents its $l_p$-norm. For a matrix $\mX$, $\sigma_i(\mX)$ is its $i$th largest singular value, or abbreviated as $\sigma_i$. We use $\mX\succeq 0$ to indicate matrix $\mX$ is positive semi-definite. $\norm{\mX}_p=(\sum_{i=1}^r \sigma_i(\mX)^p)^\frac{1}{p}$ is the Schatten-$p$ norm, where $r$ is the rank of $\mX$. Particularly, the nuclear norm is $\norm{\mX}_*=\norm{\mX}_1$, the spectral norm is $\norm{\mX}=\sigma_1(\mX)$, and the rank is equivalently represented as $\norm{\mX}_0$.
$\mX^\top$ and $\mX^\dag$ denote the transpose and the pseudo-inverse of the matrix respectively. $span(\mX)$ is the linear space spanned by the columns of $\mX$, and $P_\mX[\mmu]=\mX\mX^\dag\mmu$ denotes the projection of the vector $\mmu$ onto the column space $span(\mX)$.

\noindent\textbf{Problem statement} \quad For a given input-output pair $(\mX, \mY)$ where $\mX \in \R^{a\times b}$ and $\mY \in \R^{c \times d}$, consider the empirical loss w.r.t. a learning function $f$ parameterized by $\mtheta$, denoted as:
\begin{equation}
    \hL(\mX, \mY, \mtheta) = {l}(f(\mX;\mtheta),\mY) + \hR(\mS),  \label{eq:defopt}
\end{equation}
where $\mS=g(\mX, \mY, \mtheta) \in \R^{m \times n}.$ Here the matrix $\mS$ can depend on both the data and function parameters. 

The regularization term $\hR(\mS)$ enforces $\mS$ to be low-rank. Ideally, the definition could be directly applied, \ie, $\hR(\mS)=\lVert\mS\rVert_0$.
However, this regularization term is nondifferentiable, and the associated optimization problem is NP-hard in general.
To address this problem, in this work $\hR(\mS)$ will represent some form of approximation or relaxation of the rank.
For example, it is well-known that the nuclear norm (\ie, $\lVert\mS\rVert_*=\sum_{i=1}^r \sigma_i(\mS)$) is the convex envelope of the rank function \cite{wright2022high}, and is used as the surrogate function in many works.
However, the nuclear norm penalizes all singular values simultaneously, while in practical problems it is desired that large singular values are less impacted, so the important information of the matrix is preserved.
Motivated by this observation, many nonconvex relaxations of the rank function have been introduced \cite{kang2015robust,peng2015subspace,friedman2012fast,gao2011feasible}.
To encompass all these formulations, in this work the regularization term is considered to be of the form:
\begin{equation}
    \hR(\mS)=\sum_{i=1}^r h(\sigma_i(\mS)). \label{eq:defR}
\end{equation}
Here function $h$ generally increases the penalty of small singular values. In this paper, we refer to this form of regularization as the \textbf{generalized low-rank regularization (LRR)}.

The goal of this paper is to develop a method to approximately compute the generalized LRR $\hR(\mS)$ in Eq. (\ref{eq:defR}), which is a differentiable operation that allows the gradients to be computed and propagated.
Thus $\hR(\mS)$ can be conveniently used in a plug-and-play fashion, and the optimization problem in Eq. (\ref{eq:defopt}) can be conveniently solved by off-the-shelf optimization frameworks, and utilize high-performing GPUs for efficient parallel computation.

Several concrete applications are presented below:

\textbf{Matrix completion} \quad Suppose $\mS\in\R^{ m \times n}$ is a low-rank matrix, $\Omega\subset \{1,..,m\}\times\{1,...,n\}$ is an index set that denotes the observable entries, and $P_\Omega[\mS]_{ij}=\mathbf{1}_{\{(i,j)\in\Omega\}}\cdot \mS_{ij}$ represents the observation projection. The problem of matrix completion aims to recover $\mS$ based on partial observations, \ie, $\min_\mX \norm{P_\Omega[\mX]-P_\Omega[\mS]}_F^2 + \lambda \hR(\mX)$.

\textbf{Video fore-background separation} \quad A video sequence is represented as a 3D tensor $\mV\in\R^{a\times b \times t}$, where $a$, $b$ denote the width and height of each frame, and $t$ indexes time. Let $\mV'\in\R^{ab \times t}$ be a reshaped 2D matrix. 
In video fore-background separation it is assumed that the reshaped matrix can be decomposed as $\mV'=\mS+\mathbf{O}$, where $\mS$ is a low-rank matrix that represents the background, and $\mathbf{O}$ is a sparse matrix that represents the foreground object. So the problem can be solved by optimizing $\min_\mX \norm{\mV'-\mX}_1 + \lambda \hR(\mX).$

\textbf{DNN-based image denoising} \quad In image denoising an observed image $\mX$ is considered as a clean image $\mS$ corrupted by noise $\mathbf{N}$, \ie, $\mX=\mS+\mathbf{N}$. A DNN model learns a function $f$ to predict the noise by optimizing $\min_{\mtheta} E\left[\norm{f(\mX;\mtheta)-\mathbf{N}}_F^2\right]$.
Since the clean images are approximately low-rank,  it regularizes the loss function as
$\min_{\mtheta} E\left[\norm{f(\mX;\mtheta)-\mathbf{N}}_F^2 + \lambda \hR(f(\mX;\mtheta)-\mathbf{N})\right]$.
Note that most existing work has difficulty in optimizing such general formulations.

\subsection{Efficient Differentiable Low-rank Regularization}

\subsubsection{Iterative matrix pseudo-inverse and square root}

We first show that two fundamental operations, \ie, matrix pseudo-inverse and matrix square root, have differentiable approximations. These results serve as the cornerstones of the proposed method, as we later demonstrate that the general LRR can be constructed with these operations.

\begin{proposition}[Iterative matrix pseudo-inverse \cite{ben1966iterative}]
\label{prop:inv}
For a given matrix $\mS\in \R^{m \times n}$, define the recursive sequence $\mS_{i+1}=2 \mS_i-\mS_i \mS \mS_i$, with $\mS_0=\alpha\mS^\top$. Then $\lim_{i \to \infty} \mS_i=\mS^\dag$, provided $0<\alpha<2/\sigma_1^2(\mS)$.
\end{proposition}
Intuitively, when the iteration converges we have $\mS_i=\mS_{i+1}$. So $\mS_i=\mS_i \mS \mS_i$, which is exactly the definition of pseudo-inverse.
In practice, taking a large enough iteration step $N$, results in a satisfactory approximation of the matrix's pseudo inverse, \ie, $\mS_N \approx \mS^\dag$. Furthermore, since the iterative computation only involves matrix multiplication and subtraction, it is obviously a differentiable operation.

As a consequence, the projection operator $P_\mS[\mmu]$ also has a differentiable approximation. Recall that $P_\mS[\mmu]=\mS\mS^\dag\mmu$ denotes the projection of the vector $\mmu$ onto the column space of matrix $\mS$.
So $P_\mS[\mmu]\approx\mS\mS_N\mmu$, where $\mS_N$ is the approximation of the pseudo-inverse computed by the iterative method. Again the r.h.s. is obviously differentiable.


The matrix square root can also be computed by an iterative method, which is called the Newton-Schulz iteration:
\begin{proposition}[NS iteration for matrix square root]
\label{prop:droot}
For $\mA\in \R^{m \times m}$, initialize $\boldsymbol{Y}_0=\frac{1}{\|\boldsymbol{A}\|_{\mathrm{F}}} \boldsymbol{A}, \boldsymbol{Z}_0=\mI$. The Newton-Schulz method defines the following iteration:
$$
\boldsymbol{Y}_{k+1}=\frac{1}{2} \boldsymbol{Y}_k\left(3 \boldsymbol{I}-\boldsymbol{Z}_k \boldsymbol{Y}_k\right), \boldsymbol{Z}_{k+1}=\frac{1}{2}\left(3 \boldsymbol{I}-\boldsymbol{Z}_k \boldsymbol{Y}_k\right) \boldsymbol{Z}_k.
$$
Then $\sqrt{\|\boldsymbol{A}\|_{\mathrm{F}}} \boldsymbol{Y}_k$ quadratically converges to $\boldsymbol{A}^{\frac{1}{2}}$, \ie, $\sqrt{\|\boldsymbol{A}\|_{\mathrm{F}}} \boldsymbol{Y}_k \rightarrow \boldsymbol{A}^{\frac{1}{2}}$.
\end{proposition}

We highlight the following points of Proposition \ref{prop:inv} and \ref{prop:droot}:
1) Both are \textbf{differentiable} operations, which allow gradient-based methods for the associated optimization; 2) Both are \textbf{parallelizable} operations, which are GPU-friendly and thus efficient for large-scale datasets.
Our proposed method inherits these advantages, which results in differentiable and parallelizable approximation of the generalized LRR.

\subsubsection{Differentiable rank approximation}

The rank of a matrix is defined as the dimension of its column space (or equivalently the row space). The most well-known method for computing the rank is to first apply SVD, and then count the number of nonzero singular values. However, the SVD step is nondifferentiable. Interestingly, the following proposition presents an alternative approach for computing ranks without using SVD.
Due to space limitation, all the proofs are deferred to the Appendix.

\begin{proposition}[Equivalent definition of matrix rank \cite{wright2022high}]
\label{prop:drank}
The rank of a matrix $\mS$ can be equivalently computed as
the average squared length of a random Gaussian vector   $\mg \sim \mN(\mathbf{0},\mI)$)
projected onto its column space:
\begin{align}
    \norm{\mS}_0=rank(\mS)=E\left[\norm{P_\mS [\mg ]}_2^2\right].
\end{align}
\end{proposition}
To apply the result for practical rank computation, first sample $N$ independent random Gaussian vectors $\mg_1, ..., \mg_N \sim \mN(\mathbf{0},\mI)$. Then the sample average is used to approximate the rank, \ie, $\norm{\mS}_0=E\left[\norm{P_\mS [\mg ]}_2^2\right]\approx\frac{1}{N}\sum_{i=1}^N \norm{P_\mS [\mg_i ]}_2^2$. Since it has been shown that the projection operator $P_\mS [\mg ]$ has a differentiable approximation, by substituting the routine we obtain a differentiable method for calculating the matrix rank.

However, rank is a piecewise constant function. Though it is now differentiable, it cannot provide useful gradient information for the overall optimization problem. To address this challenge, in the following we consider several relaxations of matrix rank, the approximations of which are not only differentiable, but also provide informatic gradients.

\subsubsection{Differentiable nuclear norm approximation}

The nuclear norm is the convex envelope of rank, and is the most popular relaxation. The common approach for computing the nuclear norm is to first apply SVD, and then sum up the singular values. The following proposition presents an alternative method, similar to the case of rank calculation.

\begin{proposition}
\label{prop:dnuclear}
The nuclear norm of a matrix $\mS$ can be equivalently computed as:
\begin{align}
\norm{\mS}_*=E\left[\langle{P_\mS \left[\mg \right],(\mS\mS^\top)^\frac{1}{2}\mg }\rangle\right],
\end{align}
where $\mg \sim \mN(\mathbf{0},\mI)$ is a random Gaussian vector.
\end{proposition}
Note that besides the projection operator $P_\mS \left[\mg \right]$, this proposition further requires calculating the matrix square root $(\mS\mS^\top)^\frac{1}{2}$, which can be obtained with Proposition \ref{prop:droot}. Finally, by replacing the expectation with the sample mean, the nuclear norm can be approximately computed using differentiable operators. Furthermore, the approximated nuclear norm can provide useful gradient information for optimization.

\subsubsection{Differentiable approximation of the generalized low-rank regularization} 

In this subsection, we consider general LRR of the form $\hR(\mS)=\sum_{i=1}^r h(\sigma_i(\mS))$, where $r$ denotes the rank of $\mS$.
The function $h$ increases the penalty of the small singular values, so that the principle information of the matrix is preserved while reducing the rank. 

First, we introduce a lemma that allows any Schatten-$p$ norm to be stochastically computed, which can be viewed as an extension of Proposition \ref{prop:drank} and \ref{prop:dnuclear}.
\begin{theorem}
\label{the:pnorm}
For a matrix $\mS$, its Schatten-$p$ norm, defined as $\norm{\mS}_p=(\sum_{i=1}^r \sigma_i(\mS)^p)^\frac{1}{p}$ where $r$ is the rank of $\mS$, can be alternatively computed as
\begin{align}
\norm{\mS}_p^p=\sum_{i=1}^r \sigma_i^p=E\left[\langle{P_\mS \left[\mg \right],(\mS\mS^\top)^\frac{p}{2}\mg }\rangle\right],
\end{align}
where $p\in\mathbb{N}^+$ and $\mg \sim \mN(\mathbf{0},\mI)$.
\end{theorem}

Next we show that a broad range of relaxations of the rank function can be stochastically computed.
Particularly, we introduce two methods to utilize the above lemma, which are based on the Taylor expansion and the Laguerre expansion.

\noindent\textbf{Taylor Expansion-based Generalized LRR}
\begin{theorem}
\label{thm:taylor}
Let $\mS$ be a matrix of rank $r$, and $h:\mathbb{R}\to\mathbb{R}$ be a sufficiently smooth function and $\mg \sim \mN(\mathbf{0},\mI)$. Then the generalized LRR defined in Eq. (\ref{eq:defR}) can be computed as
\begin{align}
\sum_{i=1}^r h(\sigma_i(\mS))=\sum_{p=0}^\infty \frac{h^{(p)}(0)}{p !}E\left[\langle{P_\mS \left[\mg \right],(\mS\mS^\top)^\frac{p}{2}\mg }\rangle\right].
\end{align}
\end{theorem}

However, Taylor expansion approximates the target function based on a fixed initial point, and the truncated error grows when the evaluation location moves away from the initial point. This motivates the application of advanced approximation techniques.

\begin{figure*}[!htb]
	\centering
	\subfigure[ground truth]{
		\begin{minipage}[t]{0.2\linewidth}
			\centering
			\includegraphics[width=1.1in]{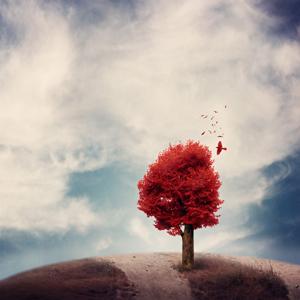}\\
		\end{minipage}%
	}%
	\subfigure[observed image]{
		\begin{minipage}[t]{0.2\linewidth}
			\centering
			\includegraphics[width=1.1in]{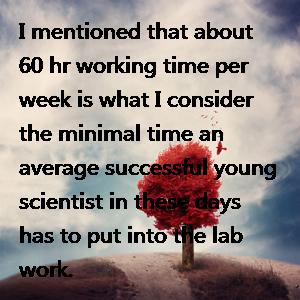}\\
		\end{minipage}%
	}%
	\subfigure[\parbox{2.2cm}{\centering RPCA (FGSR)}]{
		\begin{minipage}[t]{0.2\linewidth}
			\centering
			\includegraphics[width=1.1in]{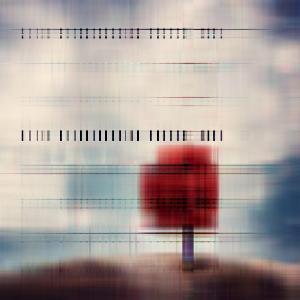}\\
		\end{minipage}%
	}%
	\subfigure[\parbox{2.4cm}{\centering fast-MDT-Tucker}]{
		\begin{minipage}[t]{0.2\linewidth}
			\centering
			\includegraphics[width=1.1in]{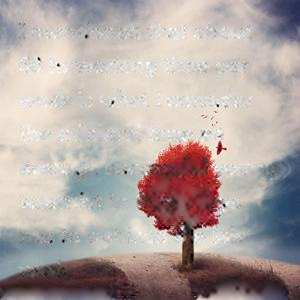}\\
		\end{minipage}%
	}%
        \subfigure[\parbox{2cm}{\centering MSS}]{
		\begin{minipage}[t]{0.2\linewidth}
			\centering
			\includegraphics[width=1.1in]{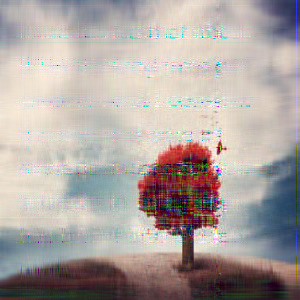}\\
		\end{minipage}%
	}%
 
        \subfigure[\parbox{2cm}{\centering IRNN}]{
		\begin{minipage}[t]{0.2\linewidth}
			\centering
			\includegraphics[width=1.1in]{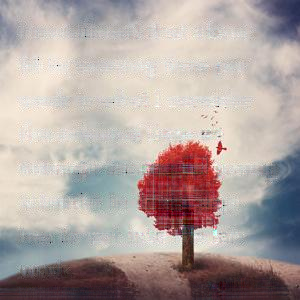}\\
		\end{minipage}%
	}%
        \subfigure[\parbox{2cm}{\centering TNN-3DTV }]{
		\begin{minipage}[t]{0.2\linewidth}
			\centering
			\includegraphics[width=1.1in]{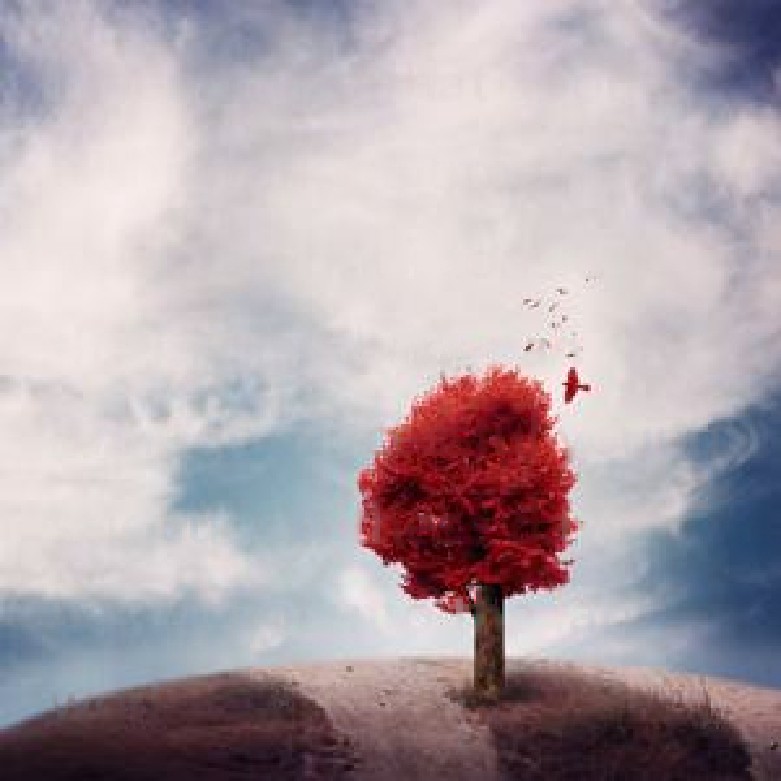}\\
		\end{minipage}%
	}%
	\subfigure[\parbox{2cm}{\centering DLRL}]{
		\begin{minipage}[t]{0.2\linewidth}
			\centering
			\includegraphics[width=1.1in]{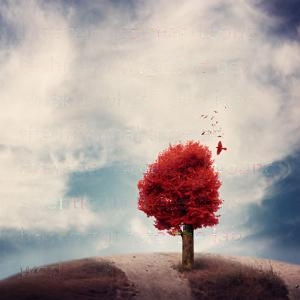}\\
		\end{minipage}%
	}%
        \subfigure[\parbox{2cm}{\centering ours-nuclear}]{
		\begin{minipage}[t]{0.2\linewidth}
			\centering
			\includegraphics[width=1.1in]{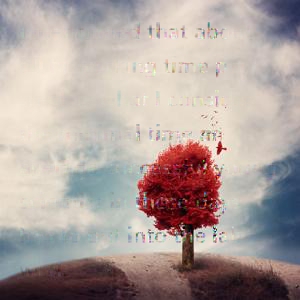}\\
		\end{minipage}%
	}%
        \subfigure[\parbox{2cm}{\centering ours-L-Lap} ]{
		\begin{minipage}[t]{0.2\linewidth}
			\centering
			\includegraphics[width=1.1in]{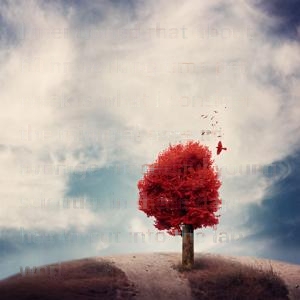}\\
		\end{minipage}%
	}%
	\centering
	\caption{Comparison of matrix completion for text removal. (a) ground truth; (b) image with text; (c)-(j) recovered images.}
	\vspace{-0.2cm}
	\label{fig:textrec0}
\end{figure*}

\noindent\textbf{Laguerre Expansion-based Generalized LRR}

\noindent Orthogonal polynomial approximation is an important family of function approximation techniques, including the well-known Legendre expansion and Laguerre expansion. In our problem the Laguerre expansion is particularly suitable, since it approximates functions with range $(0,+\infty)$.

The Laguerre expansion is based on the Laguerre polynomials $\mathfrak{L}=\{L_1(x), L_2(x), ...\}$, which is a countable infinite set of mutually orthogonal polynomials, each .denoted as $L_k(x)=\sum_p a_{k,p} x^p$.
For a target function $f(x)$ with range $(0,+\infty)$, the method decomposes it into $f(x)=\sum_{k \geq 0} c_k L_k(x)$, where $c_k$ are the coefficients computed as $c_k=\int_0^{\infty} L_k(x) e^{-x} f(x) \mathrm{d} x$.

\begin{theorem}
\label{thm:laguerre}
Keeping the same notations, we have
\begin{align*}
\sum_{i=1}^r h(\sigma_i(\mS))=\sum_{k \geq 0} \sum_p c_k a_{k,p} E\left[\langle{P_\mS \left[\mg \right],(\mS\mS^\top)^\frac{p}{2}\mg }\rangle\right].
\end{align*}
Here $h(x)=\sum_{k \geq 0} c_k L_k(x)$ utilizes the Laguerre expansion, $a_{k,p}$ are the polynomial coefficients.
\end{theorem}


In both Theorem \ref{thm:taylor} and \ref{thm:laguerre}, by using a finite sum to approximate the infinite series, the general regularization form becomes differentiable, and can provide useful gradients for the optimization.
Furthermore, the computation solely depends on matrix multiplication, which is a GPU-friendly operation that allows highly efficient parallel implementation.

\subsection{Algorithm and convergence analysis}

In what follows, we present the algorithm for computing the differentiable approximation of the Schatten-$p$ norm (Theorem \ref{the:pnorm}) and the main convergence result. Due to space limit,  detailed proofs, descriptions of applying truncated series expansion (Theorem \ref{thm:taylor} and \ref{thm:laguerre}), and further corresponding convergence analysis are left in the Appendix.

\begin{algorithm}[!t]
\caption{Differentiable approximation of $\norm{\mS}_p^p$}
\begin{algorithmic}[1]
\REQUIRE $\mS \in \R^{m \times m}$, $p$,  sample sizes ($N$), iteration steps for projection and matrix pseudo inverse ($k_1$ and $k_2$)
\ENSURE Approximation of $\norm{\mS}_p^p$
\STATE $res \gets 0$
\FOR{$i$ in $1,2,...,N$}
\STATE Sample $\mg_i \sim \mN(0,\mI)$
\STATE $\mathbf{v} \gets ApproxProject(\mS,\mg_i,k_1)$ \COMMENT{Approx. $P_\mS[\mg_i]$}
\STATE $\mathbf{M} \gets ApproxRoot(\mS\mS^\top,p,k_2)$ \COMMENT{Approx. $(\mS\mS^\top)^\frac{p}{2}$}
\STATE $res \gets res + \mathbf{v}^\top \mathbf{M} \mg_i $
\ENDFOR
\RETURN $res / N$
\end{algorithmic}
\end{algorithm}

\begin{theorem}\label{the:main_convg}
We use underline notations (\ie, $\underline{P_S[\mg_i]}$ and $\underline{(S S^\top)^\frac{p}{2}}$) to denote the approximation results obtained by the iterative methods, with iteration steps $k_1$ for matrix pseudo inverse and $k_2$ for matrix root.
Let $X=\frac{1}{N}\sum_{i=1}^N \langle \underline{P_S[\mg_i]},\underline{(S S^\top)^\frac{p}{2}}  \mg_i \rangle$. 
Then for any $\epsilon>0$,
\begin{align*}
& \Pr\left(\frac{\left|X - \sum_i \sigma_i^p \right |}{\sum_i \sigma_i^p}\leq \epsilon\right)\geq 1 - \frac{2 C(k_1, k_2)^2 \norm{S}_{2p}^{2p}}{N \left(\sum_i \sigma_i^p(\epsilon+E(k_1, k_2))\right)^2},
\end{align*}
where $C(k_1, k_2) \rightarrow 1$, $E(k_1, k_2) \rightarrow 0$ exponentially as $k_1$ and $k_2$ increase.
\end{theorem}

\noindent\textbf{Remarks:} 1) The theorem fully characterize the convergence behavior w.r.t. three determining factors, \ie, sample size ($N$), and the iteration steps ($k_1$ and $k_2$); 
2) $X$ is an unbiased estimator of the Schatten-$p$ norm; 3) $C(k_1, k_2) \rightarrow 1$ and $E(k_1, k_2) \rightarrow 0$ converges exponentially fast as $k_1$ and $k_2$ increase; 3) Larger sample size $N$ can effectively reduce the estimator's variance.

\section{Experimental Results}
In this section, we perform various experiments to demonstrate the versatility, convenience, as well as efficiency of our method. We first examine two classic LRR tasks, \ie, matrix completion and video fore-background separation. 
One advantage of our proposed method is to conveniently introduce LRR terms into any loss function, particularly deep neural networks. So we further exploit this property in DNN-based image denoising.
Experiments about convergence and parameter sensitivity is deferred to Appendix due to space constraints.
All experiments were conducted on a machine equipped with 3080Ti GPU.



\subsection{Matrix completion}


\begin{table*}[!htb]
    \centering
    \caption{Comparison of matrix completion algorithms for image inpainting. Bold \textbf{terms} and underlined \underline{terms} denote the best and second best results. ``time'' denotes the average processing time per image. See the main text for details.} \label{tab:inpaint}
    \setlength{\tabcolsep}{2.2pt}
    \begin{tabular}{l | c c c c c | c c c c c c}
      \hline
     \multirow{2}{*}{\diagbox[width=5em,trim=l]{PSNR}{Method}} & \multirow{2}{*}{nuclear} & \multirow{2}{*}{T-$\gamma^*$} & \multirow{2}{*}{T-Lap} & \multirow{2}{*}{L-$\gamma^*$} & \multirow{2}{*}{L-Lap}  & RPCA & f-MDT & \multirow{2}{*}{MSS} & \multirow{2}{*}{IRNN} & TNN- & \multirow{2}{*}{DLRL}\\
     & & & & & &FGSR&Tucker& & &3DTV& \\
      \hline
      drop 20\% & 37.09±0.18 & \underline{38.44±0.07} & 38.41±0.14 & 38.39±0.16 & \textbf{38.47±0.11} & 26.0 & 24.28 & 26.62 & 28.33 & 30.17 & 35.99 \\
      drop 30\% & 35.03±0.10 & \textbf{36.24±0.13} & 36.01±0.19 & 36.19±0.13 & \underline{36.20±0.13} & 24.89 & 24.44 & 25.58 & 28.51 & 30.03 & 34.58\\
      drop 40\% & 33.09±0.05 & \textbf{34.28±0.19} & 34.16±0.18 & 34.21±0.07 & \underline{34.22±0.15} & 16.41 & 24.16 & 24.83 & 27.52 & 29.82 & 33.58\\
      drop 50\% & 31.39±0.09 & \textbf{32.51±0.19} & 32.31±0.15 & 32.32±0.11 & 32.38±0.09 & 6.95 & 23.90 & 24.05 & 26.91 & 29.54 & \underline{32.43}\\
      block & 17.89±0.01 & \textbf{31.14±0.02} & 29.67±0.03 & 24.64±0.01 & 25.80±0.02 & 13.75 & 23.20 & 26.09 & 25.46 & 28.17 & \underline{30.82}\\
      text & 26.73±0.02 & \underline{35.07±0.06} & 34.99±0.03 & 34.97±0.05 & 34.98±0.01 & 21.49 & 22.68 & 24.56 & 26.20 & 29.99 & \textbf{37.19}\\
      \hline

      time (s) & {4.35} & {3.83} & {3.88} & {3.83} & {3.86} & 16.83 & {1.73} & 10.70 & 8.82 & 24.85 & 494.46\\
      \hline
    \end{tabular}
\end{table*}

\begin{figure*}[!htb]
	\centering
	\subfigure[escalator]{
		\begin{minipage}[t]{0.33\linewidth}
			\centering
			\includegraphics[height=1in,width=1.25in]{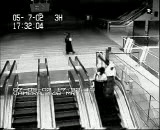}\\
			\vspace{0.02cm}
			\includegraphics[height=1in,width=1.25in]{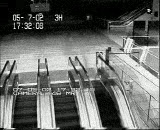}\\
			\vspace{0.02cm}
			\includegraphics[height=1in,width=1.25in]{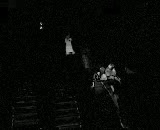}\\
			\vspace{0.1cm}
		\end{minipage}%
	}%
	\subfigure[highway]{
		\begin{minipage}[t]{0.33\linewidth}
			\centering
			\includegraphics[height=1in,width=1.25in]{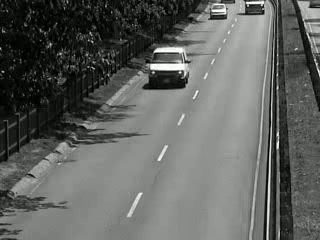}\\
			\vspace{0.02cm}
			\includegraphics[height=1in,width=1.25in]{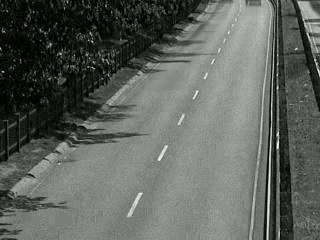}\\
			\vspace{0.02cm}
			\includegraphics[height=1in,width=1.25in]{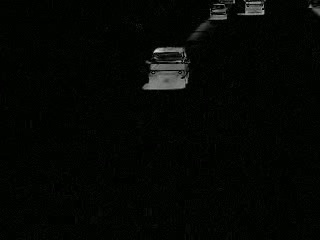}\\
			\vspace{0.1cm}
		\end{minipage}%
	}%
	\subfigure[shop]{
		\begin{minipage}[t]{0.33\linewidth}
			\centering
			\includegraphics[height=1in,width=1.25in]{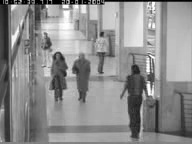}\\
			\vspace{0.02cm}
			\includegraphics[height=1in,width=1.25in]{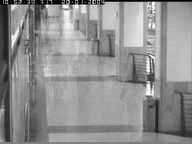}\\
			\vspace{0.02cm}
			\includegraphics[height=1in,width=1.25in]{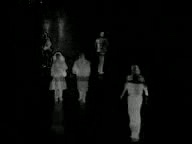}\\
			\vspace{0.1cm}
		\end{minipage}%
	}%
	\centering
	\caption{The results using our algorithm in fore-background separation. From top to bottom: original frames of the video, separated backgrounds, and foreground objects.}
	\label{fig:video}
\end{figure*}

Natural images can be represented as matrices that possess low-rank priors. 
Particularly, singular values of natural images are dominated by a few of the largest components, which allows us to model image restoration as a low-rank matrix completion problem. When dealing with colorful images, we process each color channel as an independent matrix. We evaluate the effectiveness of different methods using PSNR.

Table \ref{tab:inpaint} presents the numerical result, and Figure \ref{fig:textrec0} visualizes the image qualities. 
In Table \ref{tab:inpaint}, the first five columns denote our method, using different approximation and relaxation strategies.
Particularly, ``nuclear'' denotes using the nuclear norm without other relaxation function, \ie, directly apply Proposition \ref{prop:dnuclear}.
``T-$\gamma^*$'' and ``T-Lap'' denote using the Taylor expansion-based method for function approximation, with the $\gamma$-nuclear norm \cite{kang2015robust} and Laplace \cite{trzasko2008highly,hu2021low} as the relaxation function (\ie, $h$) respectively.
Similarly, ``L-$\gamma^*$'' and ``L-Lap'' denote using the Laguerre expansion-based method for function approximation.
``drop XX\%'' means randomly removing a certain potion of pixels, while ``block'' and ``text'' use predefined patterns to obscure the image.
The methods being compared include RPCA with FGSR \cite{fan2019factor}, f-MDT Tucker  \cite{yamamoto2022fast}, MSS \cite{oh2015partial}, IRNN \cite{lu2015nonconvex}, TNN-3DTV \cite{jiang2018anisotropic}, and DLRL \cite{chen2021efficient}.

\begin{figure*}[!htb]
	\centering
	\subfigure[\parbox{4cm}{\centering trained with noise level 15}]{
		\begin{minipage}[t]{0.33\linewidth}
			\centering
			\includegraphics[trim=120 0 0 0, clip, height = 1.75in, width=2.25in]{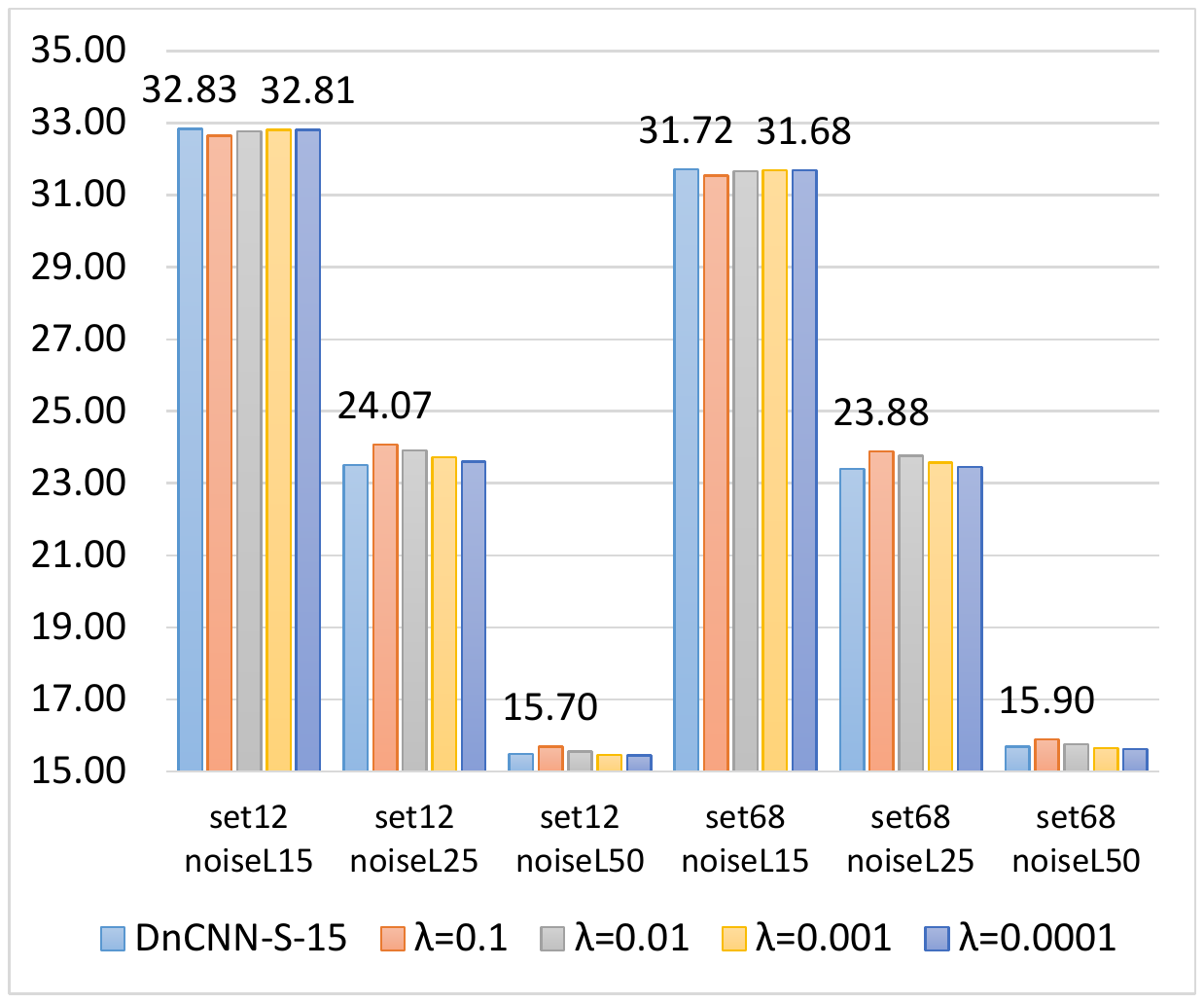}
                \vspace{-1.5cm}
		\end{minipage}
	}%
	\subfigure[\parbox{4cm}{\centering trained with noise level 25}]{
		\begin{minipage}[t]{0.33\linewidth}
			\centering
			\includegraphics[trim=120 0 0 0, clip, height = 1.75in, width=2.25in]{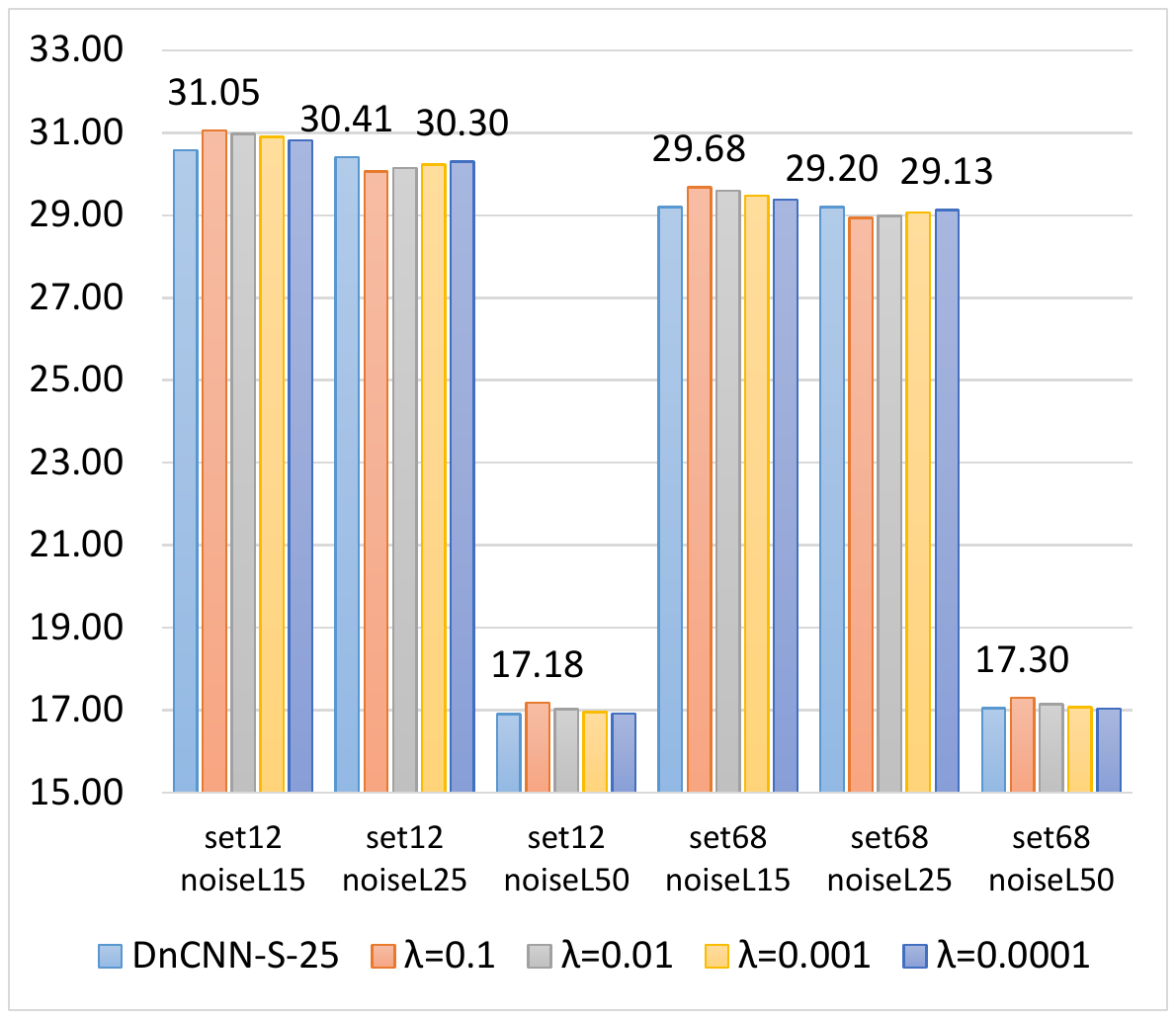}
                \vspace{-1.5cm}
		\end{minipage}
	}%
	\subfigure[\parbox{4cm}{\centering trained with noise level 50}]{
		\begin{minipage}[t]{0.33\linewidth}
			\centering
			\includegraphics[trim=120 0 0 0, clip, height = 1.75in, width=2.25in]{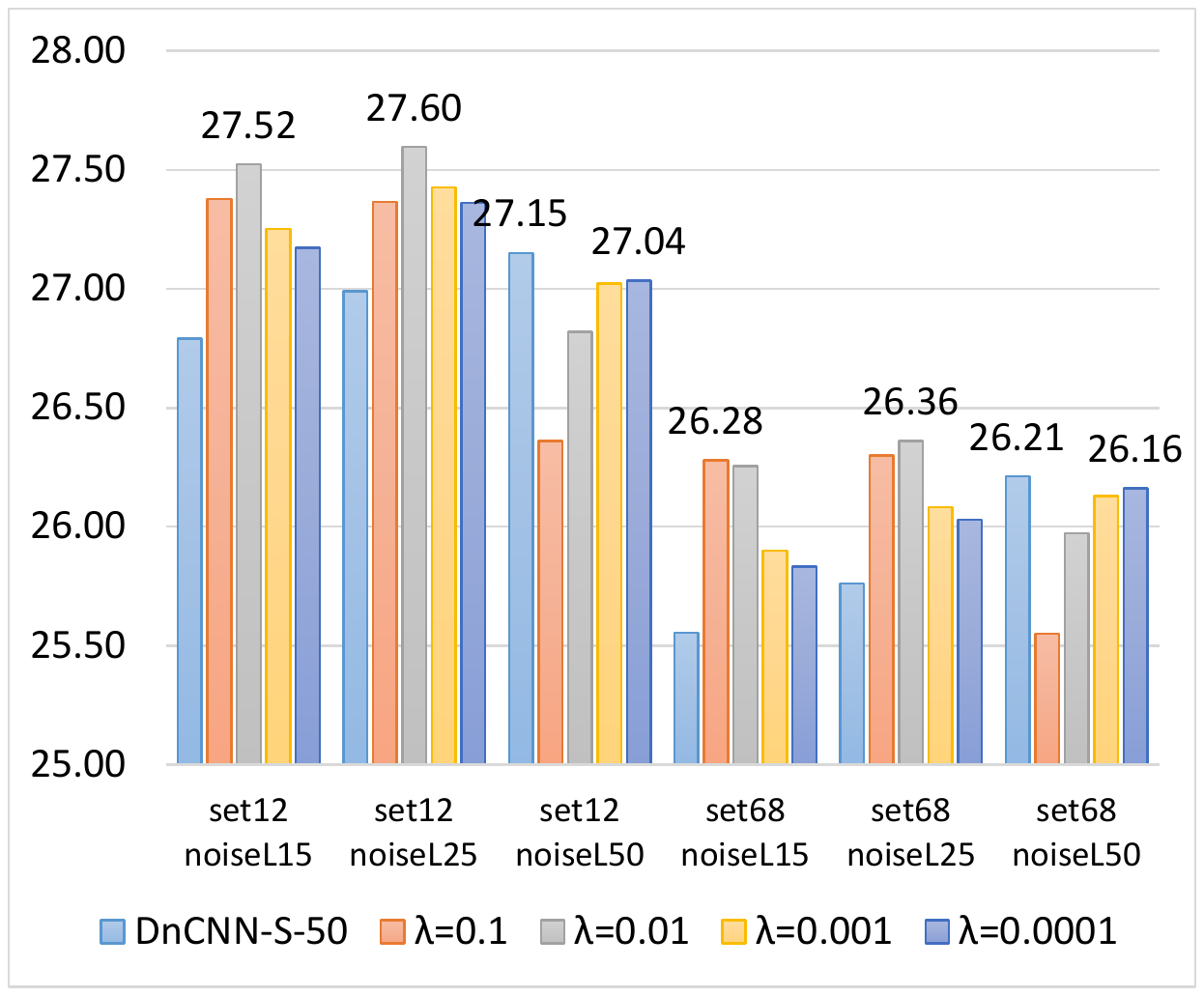}
                \vspace{-1.5cm}
		\end{minipage}
	}%
	\centering
	\caption{Results of applying low-rank regularization and the proposed differentiable approximation technique in the DnCNN denoising model, measured in PSNR.}
	\label{fig:dncnn}
\end{figure*}

From the results, we can see that: 1) Our method outperforms other baselines in almost all the cases; 2) By using relaxation methods, our performance can be further improved; 3) Both Taylor and Laguerre expansions are effective for function approximation, with each excelled in different scenarios; 4) By utilizing parallel computation and high-performing GPUs, our method strikes the best balance between performance and efficiency.

Further results and visualization of matrix completion can be found in the Appendix.


\subsection{Video fore-background separation}

We now apply our method to the task of video fore-background separation. Given a video sequence $\mathbf{V} \in \mathbb{R}^{a \times b \times t}$, where $a$ and $b$ represent the dimensions of each frame, and $t$ indexes the time steps. For each frame $\mathbf{f} \in \mathbb{R}^{a \times b}$, we can reshape the matrix into a vector $\mathbf{f'} \in \mathbb{R}^{ab \times 1}$ and concatenate all frames together, resulting in the final matrix $\mathbf{V'} \in \mathbb{R}^{ab \times t}$. 
We assume that the reshaped matrix can be decomposed as $\mathbf{V'} = \mathbf{S} + \mathbf{O}$, where $\mathbf{S}$ is a low-rank matrix representing the background, and $\mathbf{O}$ is a sparse matrix representing the foreground object. Thus, the problem can be solved by optimizing $\min_{\mathbf{X}} \|\mathbf{V'} - \mathbf{X}\|_1 + \lambda \mathcal{R}(\mathbf{X})$, where $\|\cdot\|_1$ denotes the L1 norm and $\mathcal{R}(\mathbf{X})$ represents the LRR term.
We use the nuclear norm as the surrogate rank function.


Figure \ref{fig:video} illustrates the results obtained by applying our method to fore-background separation. The algorithm effectively separates the backgrounds of individual frames by utilizing the global information of the complete video sequences. Notably, our approach produces distinct boundaries for the background, even in dynamic scenarios such as the continuously moving escalator example. Unlike conventional methods that tend to blur the details of moving objects like escalator steps, our algorithm maintains clear edges and boundaries in the obscured background region.

\subsection{Regularizing DNN-based denoising models}

One of the major strengths of our method is its flexibility in incorporating the LRR term into any loss function, and the optimization can be accomplished with deep learning libraries. Given the impressive performance of deep neural networks, it is highly desirable to apply our approach to leverage low-rank priors. We explore this avenue in DNN-based image denoising,
particularly using the denoising convolutional neural networks (DnCNNs) \cite{zhang2017beyond}.



Following the experimental configuration of \cite{zhang2017beyond}, we conducted our experiments using a training set of 400 images with dimensions of 180 × 180. Gaussian noise levels were set at $\sigma = 15, 25, 50$. Since the low-rank structure is only applicable to the entire image, we utilized the full image as input when calculating the regularization loss. For the reconstruction loss, the settings remained unchanged, with the images divided into patches of size 40 × 40. The datasets in our experiments were the Berkeley segmentation dataset (BSD68) and the Set12 dataset, which were consistent with previous studies.

The comparison between the original network and the network augmented with LRR is depicted in Figure \ref{fig:dncnn}. One of the challenges faced by these denoising networks is their reliance on prior knowledge of the noise level ($\sigma$) during training. Consequently, the trained models tend to overfit to the specific noise level provided, resulting in inferior performance when confronted with different noise levels during testing. However, as observed in the figure, this issue is significantly mitigated when the LRR and our proposed differentiable approximation are applied. Although there is a slight drop in performance when the test noise level matches the training noise level, substantial gains are observed at unseen noise levels. 
This is most evident in Figure \ref{fig:dncnn} (c), in which our method frequently improve the performance by a large margin.



\section{Conclusion}

In this paper, a novel differentiable approximation of the generalized LRR was proposed. 
The form of the regularization is quite general, covering a broad range of both convex and nonconvex relaxations.
The key advantages of the proposed method include its versatility, convenience, and efficiency.
By appending the differentiable LRR to a loss function in a plug-and-play fashion, the optimization can be automatically accomplished with gradient-based machine learning libraries.
The operations are GPU-friendly, which facilitates parallel and  efficient computation.
On the theoretical side, we rigorously prove that both the bias and the variance of the rank approximation rapidly reduce with increased sample size and iteration steps.
In the experimental study, the proposed method was successfully applied to a variety of tasks.

\newpage

\section*{Acknowledgements}
This work is supported in part by the National Natural Science Foundation of China, under Grant (62302309, 62171248), Shenzhen Science and Technology Program (JCYJ20220818101014030, JCYJ20220818101012025), the PCNL KEY project (PCL2023AS6-1), the Shenzhen Science and Technology Innovation Committee under Grant (KJZD20230923114059020, KJZD20240903103702004), and the Department of Science and Technology of Guangdong Province under Grant (2024A0101010001).

\bibliographystyle{named}
\bibliography{ijcai25}

\newpage
\section*{Appendix}

\subsection*{Proofs of differentiable low-rank regularization}
\setcounter{proposition}{2}
\begin{proposition}[Equivalent definition of matrix rank \citep{wright2022high}]
\label{prop:drank}
The rank of a matrix $\mS$ can be equivalently computed as
the average squared length of a random Gaussian vector (\ie,  $\mg \sim \mN(\mathbf{0},\mI)$) 
projected onto the column space of $\mS$:
\begin{align}
    \norm{\mS}_0=rank(\mS)=E\left[\norm{P_\mS [\mg ]}_2^2\right].
\end{align}
\end{proposition}
\begin{proof}
Suppose the rank of matrix $\mS\in \R^{m \times n}$ is $rank(\mS)=r$. By applying (compact) SVD, $\mS =\mU \mSig \mV^\top$, with $\mU \in \R^{m \times r}, \mSig \in \R^{r \times r}, \mV \in \R^{n \times r}$. Also recall that the pseudo-inverse of matrix $\mS$ satisfies $\mS ^\dag=\mV \mSig^{-1} \mU^\top$. Thus,
\begin{align*}
E\left[\norm{P_\mS \left[\mg \right]}_2^2\right]
&=E\left[(\mS \mS ^\dag\mg )^\top(\mS \mS ^\dag\mg )\right]
=E\left[\norm{\mU^\top\mg }_2^2\right]\\
&=E\left[(\mmu _1^\top\mg )^2\right]+...+E\left[({\mmu _r^\top\mg })^2\right].
\end{align*}
For any $\mmu$, we have
$
E\left[(\mmu ^\top\mg )^2\right]=\sum_{i,j=1}^m u_i u_j E\left[g_i g_j\right]=\sum_{i=1}^m u_i^2 E\left[g_i^2\right]=1.
$
So $\norm{\mS}_0=r=E\left[\norm{P_\mS [\mg ]}_2^2\right]$.
\end{proof}

\begin{proposition}
The nuclear norm of a matrix $\mS$ can be equivalently computed as:
\begin{align}
\norm{\mS}_*=E\left[\langle{P_\mS \left[\mg \right],(\mS\mS^\top)^\frac{1}{2}\mg }\rangle\right],
\end{align}
where $\mg \sim \mN(\mathbf{0},\mI)$ is a random Gaussian vector.
\end{proposition}
\begin{proof}
The proof is similar to that of Proposition \ref{prop:drank}. Keeping the same notations, it follows that
\begin{align*}
&~~~~~E\left[\langle{P_\mS \left[\mg \right],(\mS\mS^\top)^\frac{1}{2}\mg}\rangle\right]\\
&=E\left[\mg ^\top\mS \mS ^\dag(\mS\mS^\top)^\frac{1}{2}\mg \right]
=E\left[\norm{\mSig^\frac{1}{2} \mU^\top\mg }_2^2\right]\\
&=\sigma_1 E\left[(\mmu _1^\top\mg )^2\right]+...+\sigma_r E\left[({\mmu _d^\top\mg })^2\right]=\sum_{i=1}^r \sigma_i=\norm{\mS}_*. \qedhere
\end{align*}
\end{proof}

\setcounter{theorem}{0}
\begin{theorem}
\label{the:pnorm}
For a matrix $\mS$, its Schatten-$p$ norm, defined as $\norm{\mS}_p=(\sum_{i=1}^r \sigma_i(\mS)^p)^\frac{1}{p}$ where $r$ is the rank of $\mS$, can be alternatively computed as
\begin{align}
\norm{\mS}_p^p=\sum_{i=1}^r \sigma_i^p=E\left[\langle{P_\mS \left[\mg \right],(\mS\mS^\top)^\frac{p}{2}\mg }\rangle\right],
\end{align}
where $p\in\mathbb{N}^+$ and $\mg \sim \mN(\mathbf{0},\mI)$.
\end{theorem}
The proof is almost identical to the above proposition.

\begin{theorem}
Let $\mS$ be a matrix of rank $r$, and $h:\mathbb{R}\to\mathbb{R}$ be a sufficiently smooth function and $\mg \sim \mN(\mathbf{0},\mI)$. Then the generalized LRR can be computed as
\begin{align}
\sum_{i=1}^r h(\sigma_i(\mS))=\sum_{p=0}^\infty \frac{h^{(p)}(0)}{p !}E\left[\langle{P_\mS \left[\mg \right],(\mS\mS^\top)^\frac{p}{2}\mg }\rangle\right].
\end{align}
\end{theorem}
\begin{proof}
The result can be easily proved by using Theorem \ref{the:pnorm} and Taylor expansion.
\begin{align*}
\sum_{i=1}^r h(\sigma_i(\mS))=&\sum_{i=1}^r \sum_{p=0}^\infty \frac{h^{(p)}(0)}{p !}\sigma_i^p\\
=&\sum_{p=0}^\infty \frac{h^{(p)}(0)}{p !}E\left[\langle{P_\mS \left[\mg \right],(\mS\mS^\top)^\frac{p}{2}\mg }\rangle\right]. \qedhere
\end{align*}
\end{proof}

\begin{theorem}
Keeping the same notations, we have
\begin{align}
\sum_{i=1}^r h(\sigma_i(\mS))=\sum_{k \geq 0} \sum_p c_k a_{k,p} E\left[\langle{P_\mS \left[\mg \right],(\mS\mS^\top)^\frac{p}{2}\mg }\rangle\right].
\end{align}
Here $h(x)=\sum_{k \geq 0} c_k L_k(x)$ utilizes the Laguerre expansion, $a_{k,p}$ are the polynomial coefficients.
\end{theorem}
\begin{proof}
The result can be obtained by:
\begin{align*}
\sum_{i=1}^r h(\sigma_i(\mS))&=\sum_{i=1}^r \sum_{k \geq 0} c_k L_k(\sigma_i)=
\sum_{i=1}^r \sum_{k \geq 0} \sum_p c_k a_{k,p} \sigma_i^p=\\
&=\sum_{k \geq 0} \sum_p c_k a_{k,p} E\left[\langle{P_\mS \left[\mg \right],(\mS\mS^\top)^\frac{p}{2}\mg }\rangle\right].\qedhere
\end{align*}
\end{proof}

\subsection*{Algorithms and proofs of convergence}

Algorithms \ref{alg1}-\ref{alg3} summarize the procedures for appropriating Schatten-$p$ norm proposed in our work. 
Algorithms \ref{alg4} and \ref{alg5} describe computing the generalized low-rank regularization, which relaxes with a function $h$. During the calculation they require coefficients of the Taylor and Laguerre expansion. We will discuss how to conveniently obtain these coefficients with Mathematica in the next section.

\begin{algorithm}[!t]
\caption{$SchattenNorm(\mS,p,k_1,k_2)$ (Differentiable approximation of $\norm{\mS}_p^p$)}
\label{alg1}
\begin{algorithmic}[1]
\REQUIRE $\mS \in \R^{m \times m}$, $p$, sample sizes ($N$), iteration steps for projection and matrix pseudo inverse ($k_1$ and $k_2$)
\ENSURE Approximation of $\norm{\mS}_p^p$
\STATE $res \gets 0$
\FOR{$i$ in $1,2,...,N$}
\STATE Sample $\mg_i \sim \mN(0,\mI)$
\STATE $\mathbf{v} \gets ApproxProject(\mS,\mg_i,k_1)$ \COMMENT{Approx. $P_\mS[\mg_i]$}
\STATE $\mathbf{M} \gets ApproxRoot(\mS\mS^\top,p,k_2)$ \COMMENT{Approx. $(\mS\mS^\top)^\frac{p}{2}$}
\STATE $res \gets res + \mathbf{v}^\top \mathbf{M} \mg_i $
\ENDFOR
\RETURN $res / N$
\end{algorithmic}
\end{algorithm}

\begin{algorithm}[!t]
\caption{$ApproxProject(\mS,\mg,k_1)$}
\label{alg2}
\begin{algorithmic}[1]
\REQUIRE $\mS \in \R^{m \times m}$, sampled Gaussian vector $\mg$, iteration step $k_1$
\ENSURE Approximation of $P_\mS[\mg]$
\STATE Initialize $\mS_{inv}\gets\alpha\mS^\top$ \COMMENT{$\alpha$ is a sufficiently small constant s.t. $0<\alpha<2/\sigma_1^2(\mS)$}
\FOR{$k$ in $1,2,...,k_1$}
\STATE $\mS_{inv}\gets 2 \mS_{inv}-\mS_{inv}\mS\mS_{inv}$
\ENDFOR
\RETURN $\mS \mS_{inv} \mg$
\end{algorithmic}
\end{algorithm}

\begin{algorithm}[!t]
\caption{$ApproxRoot(\mS,p,k_2)$}
\label{alg3}
\begin{algorithmic}[1]
\REQUIRE $\mS \in \R^{m \times m}$, integer $p\geq 1$, iteration step $k_2$
\ENSURE Approximation of $\mS^\frac{1}{2}$
\STATE Initialize $\boldsymbol{Y}_0 \gets\frac{1}{\|\mS\|_{\mathrm{F}}} \mS, \boldsymbol{Z}_0\gets\mI$
\FOR{$k$ in $1,2,...,k_2$}
\STATE $\boldsymbol{Y}_{k}\gets\frac{1}{2} \boldsymbol{Y}_{k-1}\left(3 \boldsymbol{I}-\boldsymbol{Z}_{k-1} \boldsymbol{Y}_{k-1}\right)$
\STATE $\boldsymbol{Z}_{k} \gets \frac{1}{2}\left(3 \boldsymbol{I}-\boldsymbol{Z}_{k-1} \boldsymbol{Y}_{k-1}\right) \boldsymbol{Z}_{k-1}$
\ENDFOR
\STATE $\mathbf{R} \gets \sqrt{\|\mS\|_{\mathrm{F}}} \boldsymbol{Y}_{k_2}$
\STATE $\mathbf{R} \gets MatrixPower(\mathbf{R},p)$
\RETURN $\mathbf{R}$
\end{algorithmic}
\end{algorithm}

\begin{algorithm}[!t]
\caption{$TayLRR(\mS, h, T, k_1, k_2)$ (Taylor Expansion-based Generalized LRR)}
\label{alg4}
\begin{algorithmic}[1]
\REQUIRE $\mS \in \R^{m \times m}$, rank relaxation function $h$, truncated step $T$, sample size $N$, iteration steps $k_1, k_2$
\ENSURE Approximation of $\sum_{i=1}^r h(\sigma_i(\mS))$
\STATE Compute ${h^{(p)}(0)}$ for $p=0,1,...,T$
\STATE $res \gets 0$
\FOR{$p$ in $0, 1,2,...,T$}
\STATE $res \gets res + \frac{h^{(p)}(0)}{p !}\cdot SchattenNorm(\mS,p,k_1,k_2) $
\ENDFOR
\RETURN $res$
\end{algorithmic}
\end{algorithm}

\begin{algorithm}[!t]
\caption{$LagLRR(\mS, h, K, k_1, k_2)$ Laguerre Expansion-based Generalized LRR}
\label{alg5}
\begin{algorithmic}[1]
\REQUIRE $\mS \in \R^{m \times m}$, rank relaxation function $h$, truncated degree $K$, sample size $N$, iteration steps $k_1, k_2$
\ENSURE Approximation of $\sum_{i=1}^r h(\sigma_i(\mS))$
\STATE Perform Laguerre expansion on function $h$ truncated at degree $K$, \ie, obtain $c_k, a_{k,p}$, $ (k=0,...,K, p=0,...,k)$ s.t. $$h(x)\approx \sum_{k = 0}^K c_k L_k(x)=
\sum_{k = 0}^K \sum_{p=0}^k c_k a_{k,p} x^p$$
\STATE $res \gets 0$
\FOR{$k$ in $0, 1,2,...,K$}
\FOR{$p$ in $0, 1,2,...,k$}
\STATE $res \gets res + c_k a_{k,p}SchattenNorm(\mS,p,k_1,k_2) $
\ENDFOR
\ENDFOR
\RETURN $res$
\end{algorithmic}
\end{algorithm}

In what follows, we will prove the convergence result for Algorithm \ref{alg1}. First, we present several useful propositions and lemmas. Then we discuss the convergence property of the two sub-procedures. Finally, we will be ready for the main convergence result.

\newpage
\noindent\textbf{Auxiliary propositions and lemmas}

\begin{proposition}[Equivalence of matrix norm]\label{prop:eq_norm}
For a given matrix $\mS$, for any matrix norm $\norm{\cdot}_\alpha$ and $\norm{\cdot}_\beta$ there exists $r, s > 0$ such that
\[
r\norm{\mS}_\alpha\leq \norm{\mS}_\beta\leq s\norm{\mS}_\alpha.
\]
\end{proposition}

\begin{lemma}\label{lemma:evar}
For a given matrix $\mS$ with SVD decomposition $\mS=\mU \mSig \mV^\top$, let $\mg \sim \mN(\textbf{0}, \mI)$ be a randomly sampled Gaussian vector, then we have
\begin{itemize}
    \item $E[\mg^\top \mU \mSig^p \mU^\top \mg] = \norm{\mS}_p^p,$
    \item $\mathrm{Var}[\mg^\top \mU \mSig^p  \mU^\top \mg] = 2 \norm{\mS}_{2p}^{2p}$.
\end{itemize}
\end{lemma}
\begin{proof}
Denote $\mg'\defineas \mU^\top\mg$. Since $\mU$ is orthonormal, $\mg'$ also follows standard Gaussian distribution, \ie, $\mg' \sim \mN(\textbf{0}, \mI)$. Also note that each $g_i^{'2}$ follows the degree-1 $\chi^2$ distribution, with $E[g_i^{'2}]=1$ and $\mathrm{Var}[g_i^{'2}]=2$. Thus,
\begin{align*}
E[\mg^\top \mU \mSig^p \mU^\top \mg] =& E[\mg^{'\top} \mSig^p  \mg'] \\
=&\sum_i \sigma_i^p E[g_i^{'2}] = \sum_i \sigma_i^p\\
=& \norm{\mS}_p^p.\\
\mathrm{Var}[\mg^\top \mU \mSig^p  \mU^\top \mg]  =& \mathrm{Var}[\mg^{'\top} \mSig^p  \mg'] \\
=&\sum_i \sigma_i^{2p} \mathrm{Var}[g_i^{'2}] =2 \sum_i \sigma_i^{2p}\\
=& 2 \norm{\mS}_{2p}^{2p}.
\end{align*}
\end{proof}

\noindent\textbf{Convergence of iterative sub-procedures}

\begin{theorem}[Convergence of matrix pseudo inverse \cite{ben1966iterative}] \label{the:conv_inv} For a given matrix $\mS \in \R^{m \times m}$, let $\mX_k$ be the $k$-th step approximation of the matrix pseudo inverse computed with Line 1-4 in Algorithm \ref{alg2}. Then there exists $\alpha>0$ such that the difference between $\mS^\dag$ and $\mX_k$ is bounded by
\[
\left\|\mS^\dag-\mX_k\right\| \leq \frac{\sigma_1^{\frac{1}{2}}\left(\mS^\top \mS\right)}{\sigma_r\left(\mS^\top \mS\right)}\left(1-\alpha \sigma_r\left(\mS^\top \mS\right)\right)^{2^k},
\]
\end{theorem}

\begin{corollary} \label{coro:convg_inv}
With the same notations as in Theorem \ref{the:conv_inv}, there exist constants $K, L>0$ and $0<C, D<1$, such that 
\[
(1 - LD ^{2^k})\norm{\mS^\dag}_p \leq \norm{\mX_k}_p\leq (1+K C ^{2^k})\norm{\mS^\dag}_p.
\]
\end{corollary}
\begin{proof}
First for the RHS, by applying the sub-additivity property of matrix norm, along with Proposition \ref{prop:eq_norm} and Theorem \ref{the:conv_inv}, we have
\begin{align*}
\norm{\mX_k}_p &\leq \norm{\mS^\dag}_p + \norm{\mX_k - \mS^\dag}_p\\
&\leq \norm{\mS^\dag}_p + s \norm{\mX_k - \mS^\dag}\\
&\leq \norm{\mS^\dag}_p+ \frac{s \sigma_1^{\frac{1}{2}}\left(\mS^\top \mS\right)}{\sigma_r\left(\mS^\top \mS\right)}\left(1-\alpha \sigma_r\left(\mS^\top \mS\right)\right)^{2^k}\\
& = \left( 1 + \frac{s \sigma_1^{\frac{1}{2}}\left(\mS^\top \mS\right)}{\norm{\mS^\dag}_p\sigma_r\left(\mS^\top \mS\right)}\left(1-\alpha \sigma_r\left(\mS^\top \mS\right)\right)^{2^k}\right) \norm{\mS^\dag}_p\\
&= (1+K C ^{2^k})\norm{\mS^\dag}_p.
\end{align*}
In the last line, we define
\begin{align*}
K \defineas \frac{s \sigma_1^{\frac{1}{2}}\left(\mS^\top \mS\right)}{\norm{\mS^\dag}_p\sigma_r\left(\mS^\top \mS\right)}, \quad
C \defineas 1-\alpha \sigma_r\left(\mS^\top \mS\right).
\end{align*}
The proof of the LHS is similar.
\begin{align*}
\norm{\mX_k}_p &\geq \norm{\mS^\dag}_p - \norm{\mX_k - \mS^\dag}_p\\
&\geq \norm{\mS^\dag}_p - t \norm{\mX_k - \mS^\dag}\\
&\geq \norm{\mS^\dag}_p - \frac{t \sigma_1^{\frac{1}{2}}\left(\mS^\top \mS\right)}{\sigma_r\left(\mS^\top \mS\right)}\left(1-\alpha' \sigma_r\left(\mS^\top \mS\right)\right)^{2^k}\\
& = \left( 1 - \frac{t \sigma_1^{\frac{1}{2}}\left(\mS^\top \mS\right)}{\norm{\mS^\dag}_p\sigma_r\left(\mS^\top \mS\right)}\left(1-\alpha' \sigma_r\left(\mS^\top \mS\right)\right)^{2^k}\right) \norm{\mS^\dag}_p\\
&= (1- L D ^{2^k})\norm{\mS^\dag}_p.
\end{align*}
In the last line we define
\begin{align*}
L \defineas \frac{t \sigma_1^{\frac{1}{2}}\left(\mS^\top \mS\right)}{\norm{\mS^\dag}_p\sigma_r\left(\mS^\top \mS\right)}, \quad
D \defineas 1-\alpha' \sigma_r\left(\mS^\top \mS\right).
\end{align*}
\end{proof}

\begin{theorem}[Convergence of matrix root \cite{higham2008functions}] \label{the:conv_root} 
For a given matrix $\mS \in \R^{m \times m}$, let $\mX_k$ be the $k$-th step approximation of the matrix root computed with Line 1-6 in Algorithm \ref{alg3}. Then there exists $C>0$ such that the difference between $\mS^\frac{1}{2}$ and $\mX_k$ is bounded by
\[
\norm{X_k - \mS^\frac{1}{2}}=C \cdot 2^{-k}.
\]
\end{theorem}

\begin{corollary}\label{coro:convg_root}
With the same notations as above, there exist $C,D>0$, such that 
\[
\left(1 - D \cdot 2^{-k}\right)\norm{\mS^\frac{1}{2}}_p \leq \norm{X_k}_p \leq \left(1 +  C \cdot 2^{-k}\right)\norm{\mS^\frac{1}{2}}_p.
\]
\end{corollary}
\begin{proof}
The proof strategy is similar to that of Corollary \ref{coro:convg_inv}. Specifically, for the RHS we have
\begin{align*}
\norm{X_k}_p \leq& \norm{\mS^\frac{1}{2}}_p + \norm{X_k - \mS^\frac{1}{2}}_p\\
\leq& \norm{\mS^\frac{1}{2}}_p + s \norm{X_k - \mS^\frac{1}{2}}\\
= &\norm{\mS^\frac{1}{2}}_p + s C' \cdot 2^{-k}\\
=& \left(1 + \frac{s C' \cdot 2^{-k}}{\norm{\mS^\frac{1}{2}}_p}\right)\norm{\mS^\frac{1}{2}}_p.
\end{align*}
Similarly for the LHS, we have
\begin{align*}
\norm{X_k}_p \geq& \norm{\mS^\frac{1}{2}}_p - \norm{\mS^\frac{1}{2} - X_k}_p\\
\geq& \norm{\mS^\frac{1}{2}}_p - t \norm{\mS^\frac{1}{2} - X_k}\\
= &\norm{\mS^\frac{1}{2}}_p - t D' \cdot 2^{-k}\\
=& \left(1 - \frac{t D' \cdot 2^{-k}}{\norm{\mS^\frac{1}{2}}_p}\right)\norm{\mS^\frac{1}{2}}_p.
\end{align*}
Let $C={s C' }/{\norm{\mS^\frac{1}{2}}_p}, D={t D' }/{\norm{\mS^\frac{1}{2}}_p}$ and we obtain the result.
\end{proof}

\noindent\textbf{Main convergence result}

Now we are ready to prove the main convergence result. First we prove a simplified case, in which we assume that matrix pseudo inverse and matrix root can be exactly computed, and study the convergence behavior w.r.t. sample size.

\begin{proposition}\label{prop:main_convg}
Let $\mg_i \sim \mN(\textbf{0}, \mI)$ be independently sampled Gaussian vectors, and $X=\frac{1}{N}\sum_{i=1}^N \langle P_S[\mg_i],(S S^\top)^\frac{p}{2}  \mg_i \rangle$. Then for any $\epsilon>0$,
\begin{align}
\Pr\left(\left|X-\sum_i \sigma_i^p \right|\leq \epsilon\right)\geq 1-\frac{2\norm{S}_{2p}^{2p}}{N\epsilon^2}.
\end{align}
\end{proposition}
\begin{proof}
Let $\mS=\mU \mSig \mV^\top$ be the SVD decomposition of $\mS$.
By applying Lemma \ref{lemma:evar}, we have
\begin{align*}
E[X] =& E\left[\frac{1}{N}\sum_{i=1}^N \langle P_S[\mg_i],(S S^\top)^\frac{p}{2}  \mg_i \rangle\right]\\
=&E\left[\mg ^\top\mS \mS ^+(\mS\mS^\top)^\frac{p}{2}\mg \right]\\
=&E\left[\mg ^\top\mU \Sigma^p \mU^\top\mg \right]\\
=&\sum_i \sigma_i^p\\
\mathrm{Var}(X)&=\mathrm{Var}\left(\frac{1}{N}\sum_{i=1}^N \langle P_S[\mg_i],(S S^\top)^\frac{1}{2} \mg_i \rangle\right) \\
=&\frac{1}{N^2}\sum_{i=1}^N \mathrm{Var}\left(\langle P_S[\mg_i],(S S^\top)^\frac{1}{2} \mg_i \rangle\right) \\
=&\frac{1}{N} \mathrm{Var}\left(\mg ^\top\mU \Sigma^p \mU^\top\mg \right) \\
=&\frac{2\norm{S}_{2p}^{2p}}{N}
\end{align*}
By applying Chebyshev's inequality, we have
\begin{align*}
& \Pr\left(\left|X-E[X] \right|\geq k\sqrt{\mathrm{Var}(X)}\right)\leq \frac{1}{k^2},\\
& \Pr\left(\left|X-\sum_i \sigma_i^p \right|\geq k\sqrt{\frac{2\norm{S}_{2p}^{2p}}{N}}\right)\leq \frac{1}{k^2},\\
\end{align*}
Finally, let $\epsilon\defineas k\sqrt{{2\norm{S}_{2p}^{2p}}/{N}}$, we arrive at the desired result:
\begin{align*}
& \Pr\left(\left|X-\sum_i \sigma_i^p \right|\geq \epsilon\right)\leq \frac{2\norm{S}_{2p}^{2p}}{N\epsilon^2},\\
& \Pr\left(\left|X-\sum_i \sigma_i^p \right|\leq \epsilon\right)\geq 1- \frac{2\norm{S}_{2p}^{2p}}{N\epsilon^2}.
\end{align*}
\end{proof}

Now we will establish the main convergence result, which fully considers the approximation errors brought by the iterative computation of matrix root and pseudo inverse.
\begin{theorem}\label{the:main_convg}
We use underline notations (\ie, $\underline{P_S[\mg_i]}$ and $\underline{(S S^\top)^\frac{p}{2}}$) to denote the approximation results obtained by the iterative methods, with iteration steps $k_1$ for matrix pseudo inverse and $k_2$ for matrix root.
Let $X=\frac{1}{N}\sum_{i=1}^N \langle \underline{P_S[\mg_i]},\underline{(S S^\top)^\frac{p}{2}}  \mg_i \rangle$. 
Then for any $\epsilon>0$,
\begin{align*}
& \Pr\left(\frac{\left|X - \sum_i \sigma_i^p \right |}{\sum_i \sigma_i^p}\leq \epsilon\right)\geq 1 - \frac{2 C(k_1, k_2)^2 \norm{S}_{2p}^{2p}}{N \left(\sum_i \sigma_i^p(\epsilon+E(k_1, k_2))\right)^2},
\end{align*}
where $C(k_1, k_2) \rightarrow 1$, $E(k_1, k_2) \rightarrow 0$ exponentially as $k_1$ and $k_2$ increase.
\end{theorem}
\begin{proof}
The proof framework is similar to that of Proposition \ref{prop:main_convg}, in which we plugin the convergence results of the iterative sub-procedures.

First, we develop the lower and upper bounds of $E[X]$. Specifically,
\begin{align*}
E[X] = E[\mg ^\top\mS \underline{\mS^\dag}  (\underline{(SS^\top)^\frac{1}{2}})^p\mg]= \norm{\mS \underline{\mS^\dag}  (\underline{(SS^\top)^\frac{1}{2}})^p}_1.
\end{align*}
From  Corollary \ref{coro:convg_inv} and \ref{coro:convg_root} , there exists $K, L>0$, $0<C_1, D_1<1$ and $C_2, D_2>0$, 
such that
\begin{align*}
E[X] &\leq  (1+K C_1^{2^{k_1}})(1+C_2 \cdot 2^{-k_2}) \norm{\mS {\mS} ^\dag ({(SS^\top)^\frac{1}{2}})^p}_1\\
&= C(k_1, k_2) \sum_i \sigma_i^p,\\
E[X] &\geq  (1-L D_1^{2^{k_1}})(1-D_2 \cdot 2^{-k_2}) \norm{\mS {\mS} ^\dag ({(SS^\top)^\frac{1}{2}})^p}_1\\
&= D(k_1, k_2) \sum_i \sigma_i^p,\\
\end{align*}
here we define 
\begin{align*}
C(k_1, k_2) \defineas (1+K C_1^{2^{k_1}})(1+C_2 \cdot 2^{-k_2}),\\
D(k_1, k_2) \defineas (1-L D_1^{2^{k_1}})(1-D_2 \cdot 2^{-k_2}).
\end{align*}
Let $E(k_1, k_2)\defineas max\{C(k_1, k_2) - 1, 1-D(k_1, k_2)\}$, we have
\begin{align*}
&D(k_1, k_2) - 1 \leq \frac{ E[X] -  \sum_i \sigma_i^p}{\sum_i \sigma_i^p} \leq  C(k_1, k_2) - 1,\\
&-E(k_1, k_2)  \leq \frac{ E[X] -  \sum_i \sigma_i^p}{\sum_i \sigma_i^p} \leq  E(k_1, k_2) ,\\
&\frac{ \left|{E[X] -  \sum_i \sigma_i^p}\right|}{\sum_i \sigma_i^p} \leq  E(k_1, k_2).
\end{align*}
The upper bound of the variance is derived as
\begin{align*}
\mathrm{Var}[X] 
&= \mathrm{Var}[\mg ^\top\mS \underline{\mS^\dag}  (\underline{(SS^\top)^\frac{1}{2}})^p\mg]\\
&= \frac{2}{N} \norm{\mS \underline{\mS^\dag}  (\underline{(SS^\top)^\frac{1}{2}})^p}_2^2\\
&\leq \frac{2 (1+K C_1^{2^{k_1}})^2(1+C_2 \cdot 2^{-k_2})^2 }{N}\norm{\mS {\mS} ^\dag ({(SS^\top)^\frac{1}{2}})^p}_2^2\\
&=\frac{2 C(k_1, k_2)^2}{N} \norm{\Sigma^p}_2^2 =\frac{2 C(k_1, k_2)^2}{N} \norm{S}_{2p}^{2p}.
\end{align*}

From Chebyshev's inequality,
\begin{align*}
& \Pr\left(\left|X- E(X) \right|\geq k C(k_1, k_2) \sqrt{\frac{2 \norm{S}_{2p}^{2p}}{N}}\right)\leq {\frac {1}{k^{2}}},\\
& \Pr\left(\frac{\left|X- E(X) \right|}{\sum_i \sigma_i^p} \geq \frac{k C(k_1, k_2)}{\sum_i \sigma_i^p} \sqrt{\frac{2 \norm{S}_{2p}^{2p}}{N}}\right)\leq {\frac {1}{k^{2}}}.
\end{align*}
Notice that
\begin{align*}
\frac{\left|X- E(X) \right|}{\sum_i \sigma_i^p} &\leq
\frac{\left|X - \sum_i \sigma_i^p \right | + \left | \sum_i \sigma_i^p- E(X) \right|}{\sum_i \sigma_i^p} \\
&\leq \frac{\left|X - \sum_i \sigma_i^p \right |}{\sum_i \sigma_i^p} + E(k_1, k_2).
\end{align*}
Let $\epsilon=\frac{k C(k_1, k_2)}{\sum_i \sigma_i^p} \sqrt{\frac{2 \norm{S}_{2p}^{2p}}{N}} - E(k_1, k_2)$,
we finally have
\begin{align*}
&\Pr\left(\frac{\left|X - \sum_i \sigma_i^p \right |}{\sum_i \sigma_i^p}\geq \epsilon\right)\leq  \frac{2 C(k_1, k_2)^2 \norm{S}_{2p}^{2p}}{N \left(\sum_i \sigma_i^p(\epsilon+E(k_1, k_2))\right)^2},\\
&\Pr\left(\frac{\left|X - \sum_i \sigma_i^p \right |}{\sum_i \sigma_i^p}\leq \epsilon\right)\geq \\
& ~~~~~~~~~~~~~~~~~~ 1 - \frac{2 C(k_1, k_2)^2 \norm{S}_{2p}^{2p}}{N \left(\sum_i \sigma_i^p(\epsilon+E(k_1, k_2))\right)^2}.
\end{align*}
\end{proof}

\noindent\textbf{Convergence with truncated Taylor's expansion}

Next we present the convergence result of generalized LRR, which use a function $h$ to penalizes small singular values. First, we present the well-known convergence result of Taylor's expansion in Proposition \ref{prop:taylor}. Then the convergence of applying Taylor's expansion is shown in Proposition \ref{prop:taylor_convg2}.

\begin{proposition}[Mean-value form of Taylor's theorem] \label{prop:taylor} For a $T+1$ times differentiable function $h$, define the reminder term of the Taylor's expansion truncated at degree $T$ as $R_T(x) \defineas h(x) - \sum_{p=0}^T\frac{h^{(p)}(a)}{p!}(x-a)^p$.
Then there exists $a\leq \xi \leq x$ such that
\[
R_{T}(x)={\frac {f^{(T+1)}(\xi)}{(T+1)!}}(x-a)^{T+1}.
\]
\end{proposition}

\begin{proposition}
    \label{prop:taylor_convg2}
Let $Z=TayLRR(\mS, h, T, k_1, k_2)$ be the result computed by Algorithm \ref{alg4}, which approximates $\sum_{i=1}^r h(\sigma_i(\mS))$. 
Following previous notations, \eg, $X_p=\frac{1}{N}\sum_{i=1}^N \langle \underline{P_S[\mg_i]},\underline{(S S^\top)^\frac{p}{2}}  \mg_i \rangle$. 
Then for any $\epsilon>0$,
\begin{align*}
& \Pr\left(\left|\sum_{i=1}^r h(\sigma_i(\mS)) - Z \right| \leq \epsilon' \right) \geq \\
& ~~~~~~~~~~~~~~~~ 1 - \frac{2 C(k_1, k_2)^2 \norm{S}_{2p}^{2p}}{N \left(\sum_i \sigma_i^p(\epsilon+E(k_1, k_2))\right)^2},
\end{align*}
where 
\begin{align*}
&\epsilon'=\epsilon \sum_{i=0}^r \underline{h_T(\sigma_i)} + \sum_{i=1}^r R_T(\sigma_i),\\
&\underline{h_T(\sigma_i)} = \sum_{p=0}^T \frac{h^{(p)}(0)}{p!}\sigma_i^p.
\end{align*}
Here $R_T(\sigma_i)$ is the reminder defined in Proposition \ref{prop:taylor}.
\end{proposition}
\begin{proof}
First, note that $Z=TayLRR(\mS, h, T, k_1, k_2)$ essentially computes
\[
\sum_{p=0}^T \frac{h^{(p)}(0)}{p !}X_p.
\]
From Theorem \ref{the:main_convg}, with probability at least $1 - \frac{2 C(k_1, k_2)^2 \norm{S}_{2p}^{2p}}{N \left(\sum_i \sigma_i^p(\epsilon+E(k_1, k_2))\right)^2}$, we have $\left|X_p-\sum_i \sigma_i^p \right| \leq \epsilon \sum_i \sigma_i^p.$ Thus, with the same probability, we have
\begin{align*}
& \left| \sum_{i=1}^r h(\sigma_i(\mS))-Z\right | \\
=& \left| \sum_{i=1}^r h(\sigma_i(\mS))-\sum_{p=0}^T \frac{h^{(p)}(0)}{p !}X_p\right | \\
=& \left| \sum_{i=1}^r \left( \sum_{p=0}^T\frac{h^{(p)}(0)}{p!}\sigma_i^p + R_T(\sigma_i) \right)-\sum_{p=0}^T \frac{h^{(p)}(0)}{p !}X_p\right | \\
\leq & \left| \sum_{i=1}^r \sum_{p=0}^T\frac{h^{(p)}(0)}{p!}\sigma_i^p-\sum_{p=0}^T \frac{h^{(p)}(0)}{p !}X_p\right |  + \sum_{i=1}^r R_T(\sigma_i)\\
= & \sum_{p=0}^T \frac{h^{(p)}(0)}{p!}\left| \sum_{i=1}^r\sigma_i^p- X_p\right |  + \sum_{i=1}^r R_T(\sigma_i)\\
\leq & \sum_{p=0}^T \frac{h^{(p)}(0)}{p!} \epsilon \sum_i \sigma_i^p + \sum_{i=1}^r R_T(\sigma_i)\\
= & \epsilon \sum_{i=0}^r \underline{h_T(\sigma_i)} + \sum_{i=1}^r R_T(\sigma_i). \qedhere
\end{align*}
\end{proof}

\subsection*{Compute Taylor and Laguerre expansion with Mathematica}


\begin{figure*}[!t]
    \centering
    \includegraphics[width=1\linewidth]{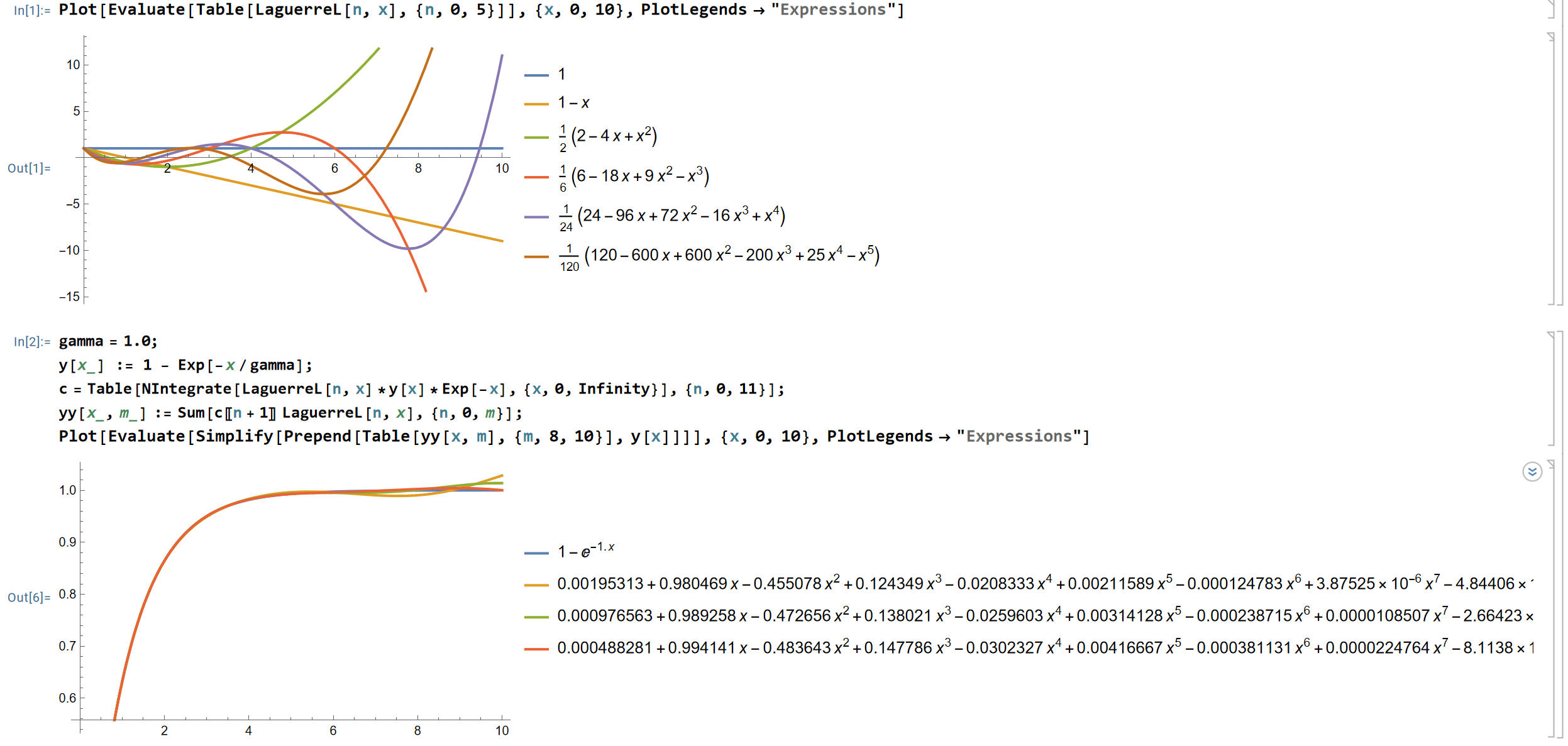}
    \caption{Approximating the Laplace relaxation function with Laguerre's expansion, with degree 8-10 polynomials. 1) Visualization of the standard degree 0-5 Laguerre polynomials; 2) Visualization of the approximation quality with degree 8-10 polynomials.`}
    \label{fig:lag_lap}
\end{figure*}

The coefficients of the series expansion required in Algorithms \ref{alg4} and \ref{alg5} can conveniently computed by software like Mathematica. 
We demonstrate an example of approximating the Laplace relaxation.

\noindent \textbf{Laplace relaxation} \citep{trzasko2008highly,hu2021low}:
The Laplace relaxation penalizes small singular values with a function $h$, defined as
$h(\sigma)=1-\exp{(-\sigma/\gamma)}$. Note that $h$ essentially puts more weights on small singular values, so they will be more rapidly reduced during optimization.

The code and results are presented in  Figure \ref{fig:lag_lap} (Laguerre's expansion).

Remarks: 1) The implementation is quite convenient, requiring only a few lines of code. 2) The computational time is negligible. In our experiments all expansion coefficients can be instantly computed. 3) Empirically, Laguerre's expansion has better approximation quality with fewer terms.

\begin{figure*}[!t]
	\centering
	\subfigure[]{
		\begin{minipage}[t]{0.24\linewidth}
			\centering
			\includegraphics[width=1.38in]{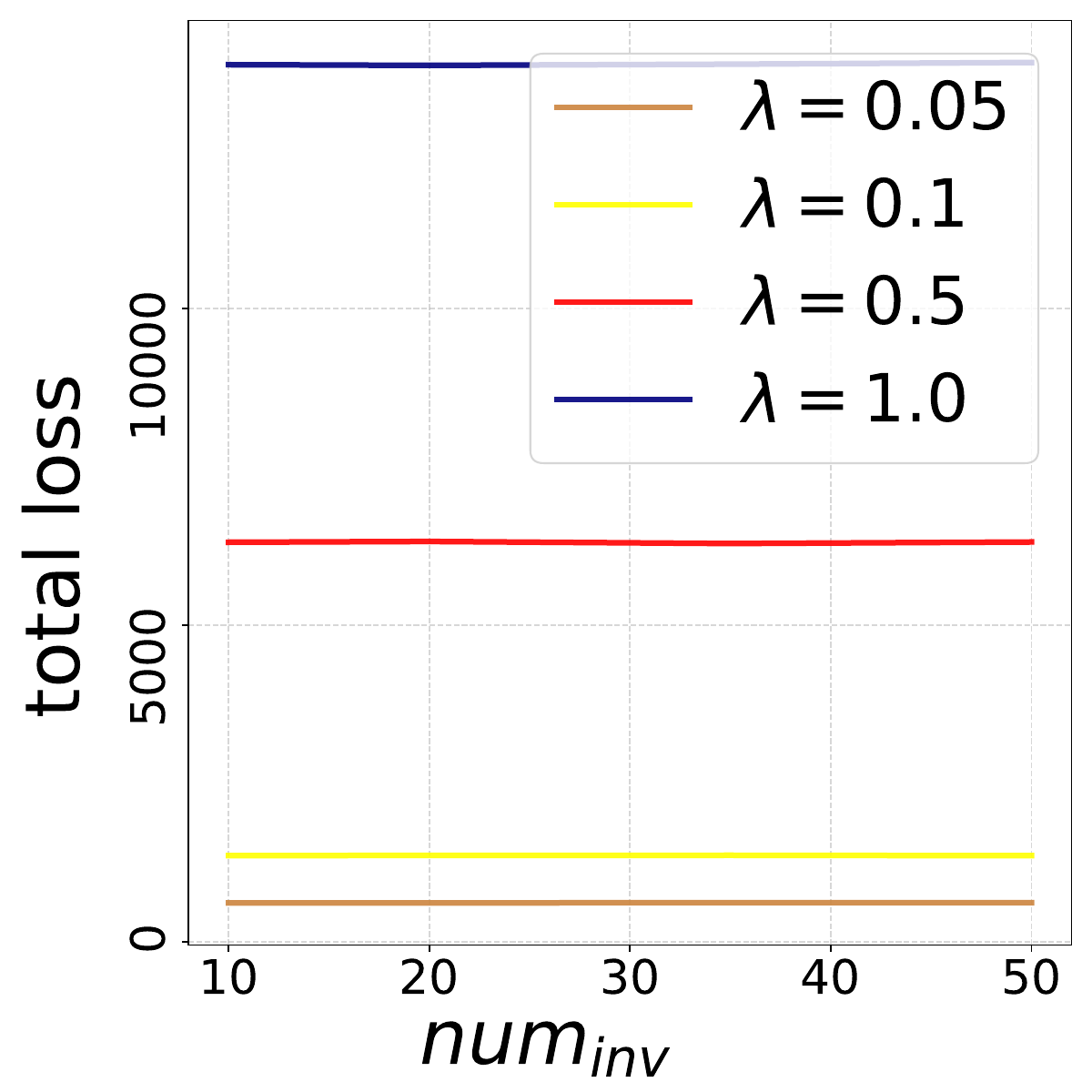}
		\end{minipage}
	}%
	\subfigure[]{
		\begin{minipage}[t]{0.24\linewidth}
			\centering
			\includegraphics[width=1.38in]{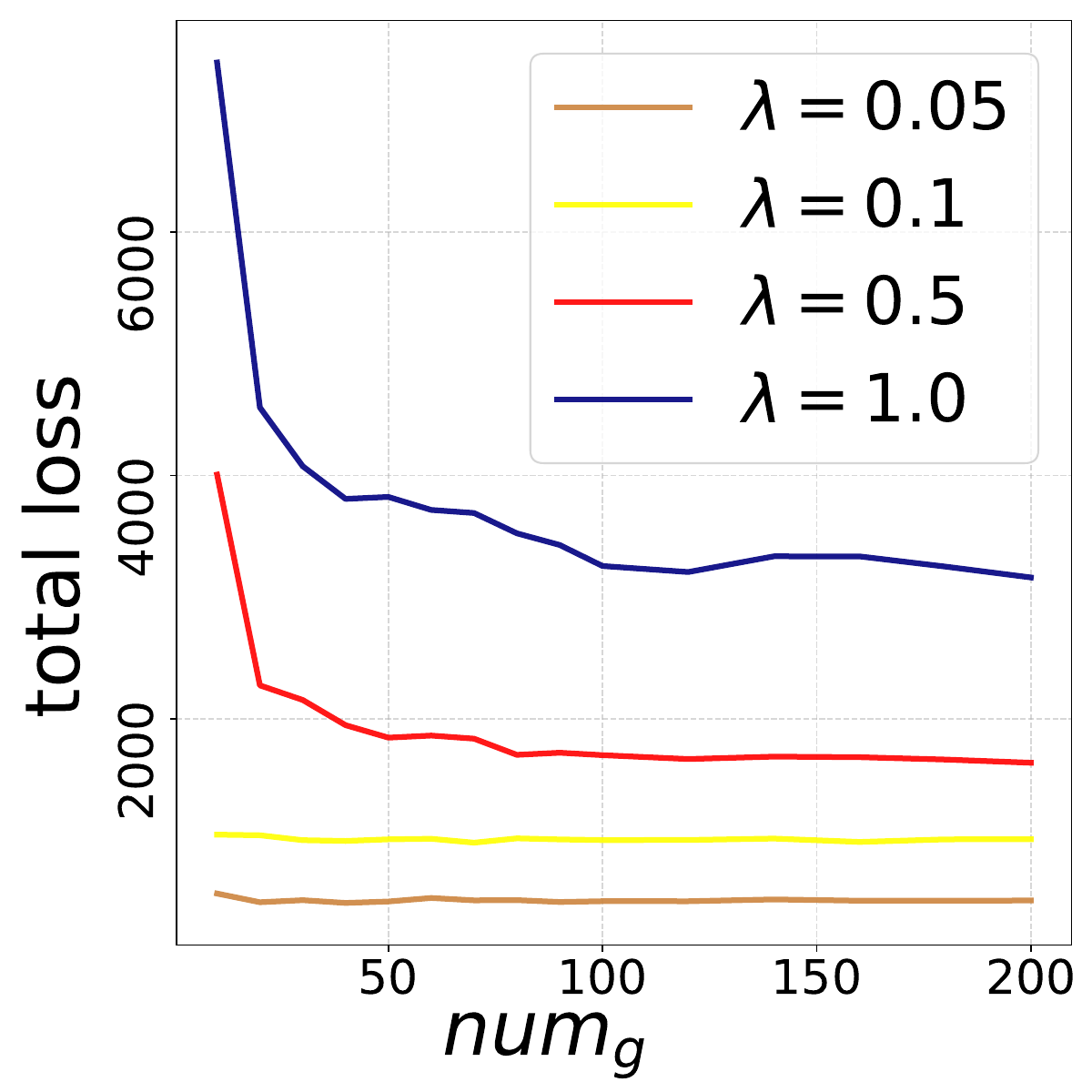}
		\end{minipage}
	}%
	\subfigure[]{
		\begin{minipage}[t]{0.24\linewidth}
			\centering
			\includegraphics[width=1.38in]{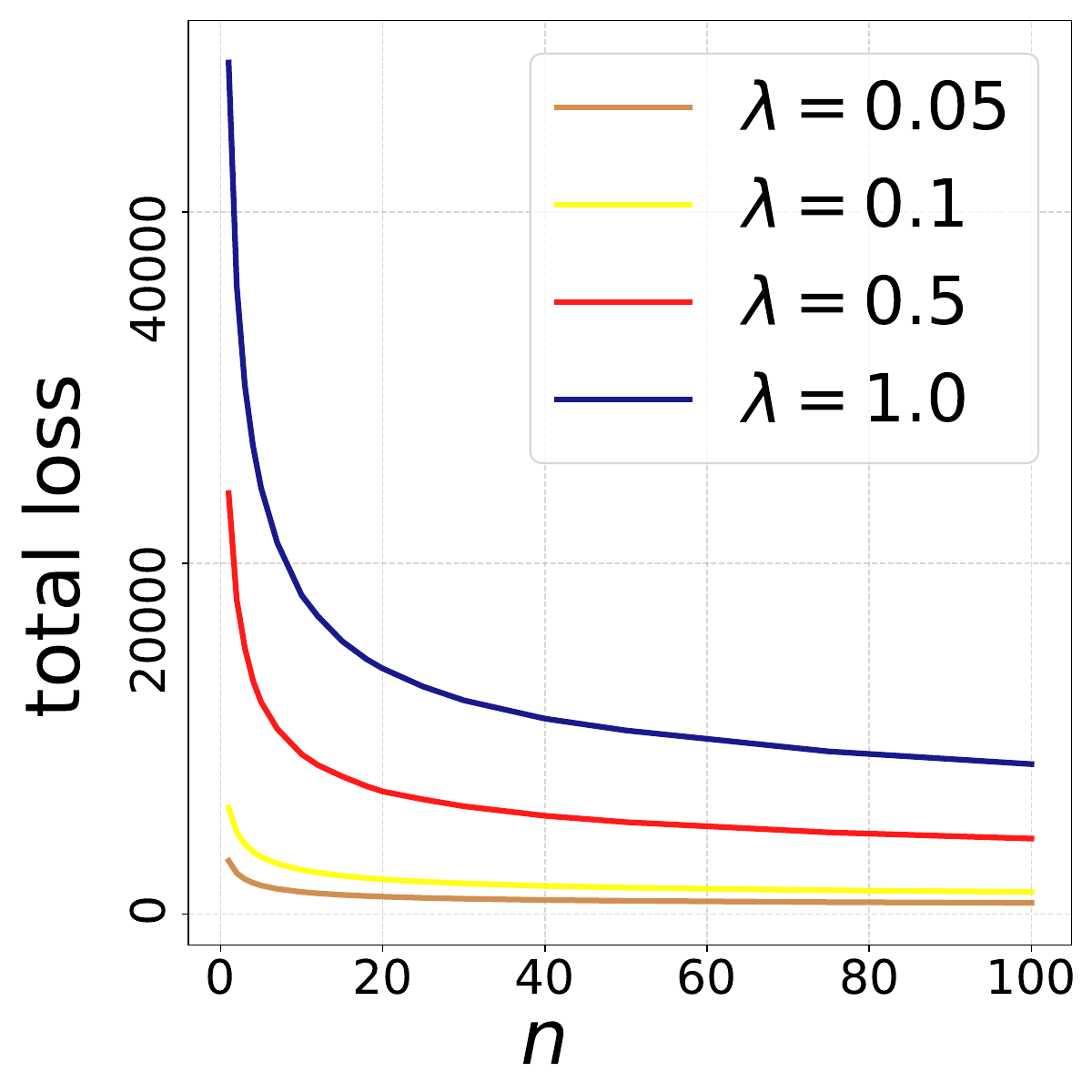}
		\end{minipage}
	}%
	\subfigure[]{
		\begin{minipage}[t]{0.27\linewidth}
			\centering
			\includegraphics[height=1.38in, width=1.56in]{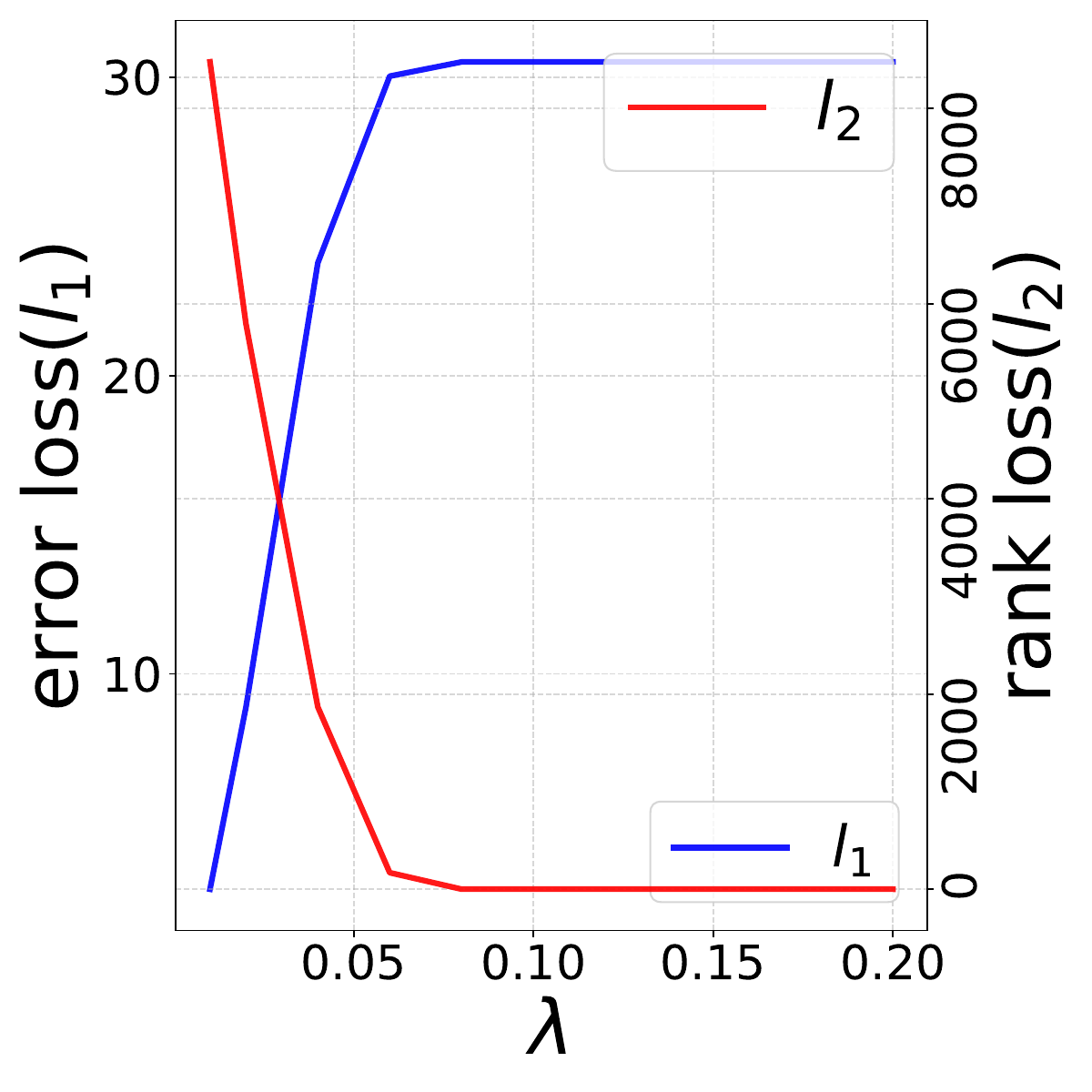}
		\end{minipage}
	}%
	\centering
    \caption{Numerical analysis on the synthetic dataset.}
	\label{fig:toy}
\end{figure*}

\begin{figure*}[!t]
	\centering
	\subfigure[observed image]{
		\begin{minipage}[t]{0.2\linewidth}
			\centering
			\includegraphics[width=1.1in]{baseline_compare/X_obs.jpg}\\
		\end{minipage}%
	}%
        \subfigure[\parbox{2.4cm}{\centering RPCA\\(nuclear norm)}]{
		\begin{minipage}[t]{0.2\linewidth}
			\centering
			\includegraphics[width=1.1in]{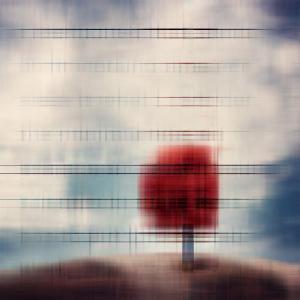}\\
		\end{minipage}%
	}%
	\subfigure[\parbox{2.2cm}{\centering RPCA\\(F-nuclear norm)}]{
		\begin{minipage}[t]{0.2\linewidth}
			\centering
			\includegraphics[width=1.1in]{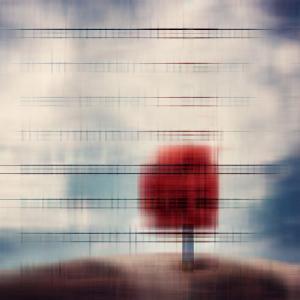}\\
		\end{minipage}%
	}%
        \subfigure[\parbox{2cm}{\centering SNN}]{
		\begin{minipage}[t]{0.2\linewidth}
			\centering
			\includegraphics[width=1.1in]{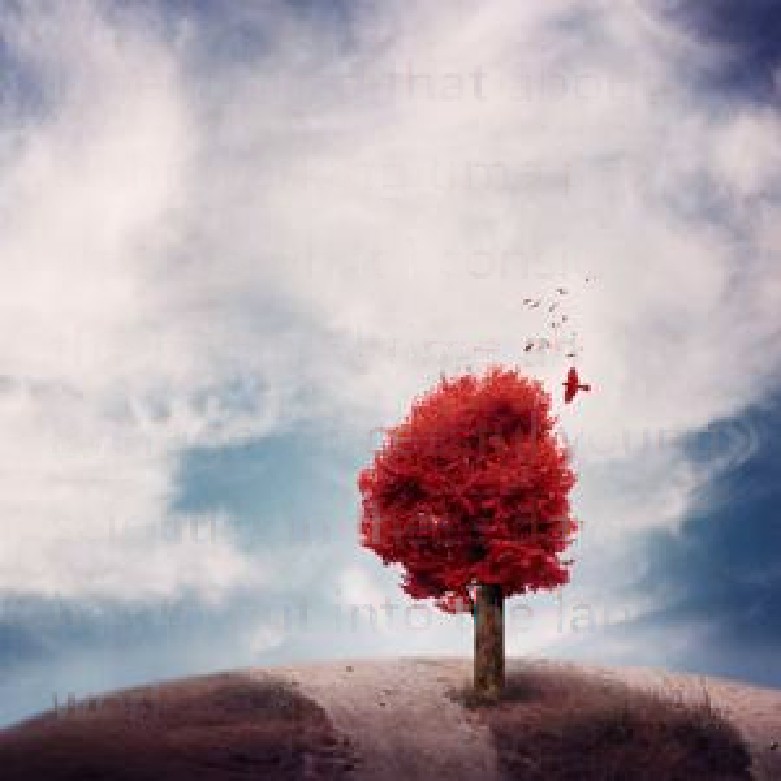}\\
		\end{minipage}%
	}%
        \subfigure[\parbox{2cm}{\centering TNN}]{
		\begin{minipage}[t]{0.2\linewidth}
			\centering
			\includegraphics[width=1.1in]{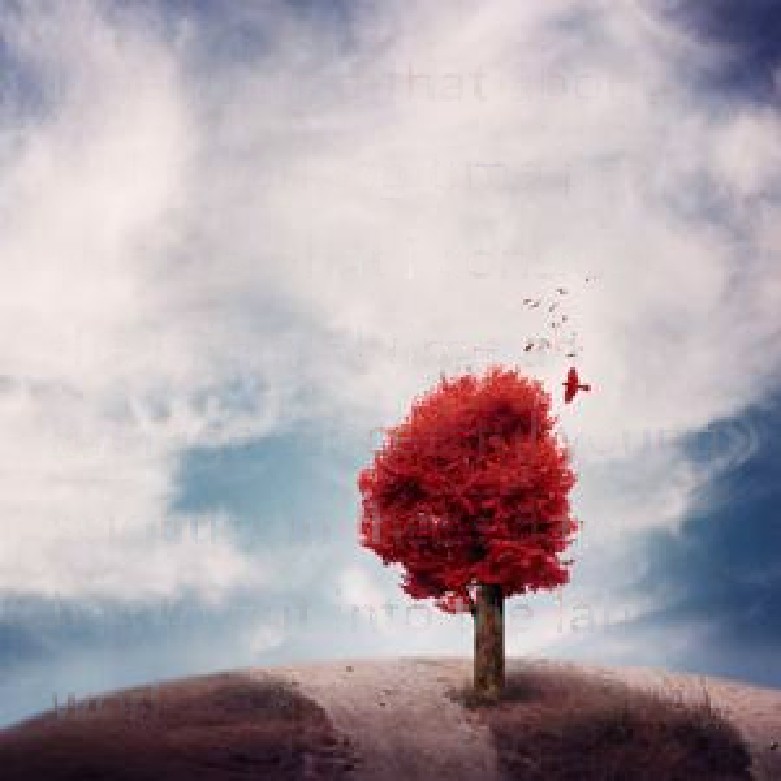}\\
		\end{minipage}%
	}%
	\centering
	\caption{More text inpaiting results. (a) image with text (b)-(e) recovered images.}
	\vspace{-0.2cm}
	\label{fig:textrec1}
\end{figure*}


\subsection*{Convergence and parameter sensitivity analysis}
\label{app:para}

We perform experiments on a synthetic dataset, the goal of which is to analyze the convergence and parameter sensitivity of the proposed method.
The synthetic data is construed by sampling $\mA \in \R^{m \times r}$, $\mB \in \R^{r \times n}$, and $\mE \in \R^{m \times n}$, the entries of each matrix follow i.i.d. Gaussian distribution ($m=n=30, r=30)$.
We define $\mC = \mA  \mB$, $\mS = \mC + \sigma \mE$, where $\mC$ is the underlying low-rank matrix need to be discovered, and $\mS$ represents the observed full-rank matrix obstructed by the noise $\sigma \mE$.
With this formulation, the target low-rank matrix can be found by solving:
\begin{align}
l=\operatorname*{min}_{\mX}\underbrace{\|\mS-\mX\|_F^2}_{l_1\text{: reconstruction loss}}+\underbrace{\lambda \hR(\mX)}_{l_2\text{: low-rank regularization}}.
\end{align}
Here we decompose the total loss $l$ into the reconstruction loss $l_1$ and the regularization loss $l_2$, and $\lambda$ controls the strength of regularization.
In this experiment we consider the regularization to be the nuclear norm, \ie, $\hR(\mX)=\|\mX\|_*$.
In our method, the following parameters play significant roles:

$num_{inv}$: In Proposition 1, the pseudo-inverse of a matrix is approximated by finite iterations. The iteration step is denoted as $num_{inv}$.\\
$num_g$: In Proposition 2, the matrix square root is approximated by a finite power series. The number of the summation terms is denoted as $num_g$.\\
$n$: In Theorem 1, the expectation will be approximated by an average of finite samples.
The number of samples is denoted as $n$.


The sensitivity analysis results of these key parameters are shown in Figure \ref{fig:toy} (a)-(c), from which we can see that: 1) Surprisingly, the choice of $num_{inv}$ has little impact on the convergence, and 10 iterations are sufficient for approximating the pseudo-inverse. A probable explanation is that gradient-based optimization involves many iteration steps, so its requirement for per-step accuracy is not strict.
2) The convergence is quite stable w.r.t. $num_g$ and $n$, as long as their values are reasonably large ($num_g\geq 100, n\geq 80$).
We further vary parameter $\lambda$ to control the strength of regularization. As shown in Figure \ref{fig:toy} (d), this parameter can effectively control the tradeoff between the reconstruction loss and the regularization loss as expected.

\subsection*{More results for matrix completion}


\begin{table}[!t]
    \centering
    \caption{More results for comparison of matrix completion algorithms for image inpainting.}\label{tab:inpaint1}
    \small
    \setlength{\tabcolsep}{4pt}
    \begin{tabular}{l | c c c c c c c c c c c}
    
      \hline
     \multirow{2}{*}{\diagbox[width=6em,trim=l]{PSNR}{Method}} & RPCA & RPCA & \multirow{2}{*}{SNN} & \multirow{2}{*}{TNN} \\
      &nuclear norm&F-nuclear norm& & \\

      \hline
      drop 20\% & 27.50 & 27.49 & 30.14 & 30.09\\
      drop 30\% & 23.70 & 23.70 & 29.95 & 29.87\\
      drop 40\% & 12.09 & 12.08 & 29.68 & 29.56\\
      drop 50\% & 6.40 & 6.40 & 29.23 & 29.03\\
      block & 13.73 & 13.73 & 27.54 & 27.84\\
      text & 22.25 & 22.26 & 29.81 & 29.81\\
      \hline
      time (s) & 27.64 & 5.16 & 6.73 & 9.08\\
      \hline
    \end{tabular}
\end{table}

Table \ref{tab:inpaint1} and Figure \ref{fig:textrec1} present more results of matrix completion algorithms for image inpaiting task. 
Specifically, RPCA (nuclear norm) and RPCA (F-nuclear norm) are weaker versions of RPCA (FGSR) \citep{fan2019factor}. Sum of the Nuclear Norm (SNN) \citep{liu2012tensor} and Tensor Nuclear Norm (TNN) \citep{zhang2014novel} are weaker versions of TNN-3DTV.

\begin{figure*}[t]
	\centering
	\subfigure[\parbox{4cm}{\centering observed image}]{
		\begin{minipage}[t]{0.33\linewidth}
			\centering
			\includegraphics[width=1.7in]{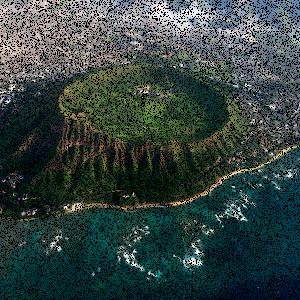}
                \vspace{0.5cm}
		\end{minipage}
	}%
	\subfigure[\parbox{4cm}{\centering recovered by ours-nuclear\\PSNR=28.45, SSIM=0.9231}]{
		\begin{minipage}[t]{0.33\linewidth}
			\centering
			\includegraphics[width=1.7in]{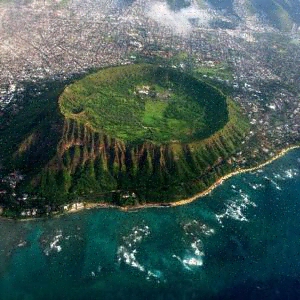}
                \vspace{0.3cm}
		\end{minipage}
	}%
	\subfigure[\parbox{4cm}{\centering recovered by ours-L-Lap \\ PSNR=29.92, SSIM=0.9427}]{
		\begin{minipage}[t]{0.33\linewidth}
			\centering
			\includegraphics[width=1.7in]{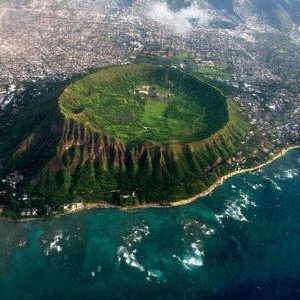}
                \vspace{0.3cm}
		\end{minipage}
	}%
	\centering
	\caption{Results of applying our method with nuclear norm and $h$-functions to an image randomly blocked by 25\%.}
	\label{fig:1}
\end{figure*}

\begin{figure*}[!t]
	\centering
	\subfigure[\parbox{4cm}{\centering observed image}]{
		\begin{minipage}[t]{0.33\linewidth}
			\centering
			\includegraphics[width=1.7in]{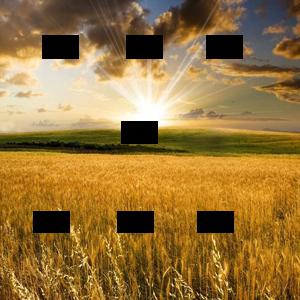}
                \vspace{0.5cm}
		\end{minipage}
	}%
	\subfigure[\parbox{4cm}{\centering recovered by ours-nuclear\\PSNR=26.62, SSIM=0.9694}]{
		\begin{minipage}[t]{0.33\linewidth}
			\centering
			\includegraphics[width=1.7in]{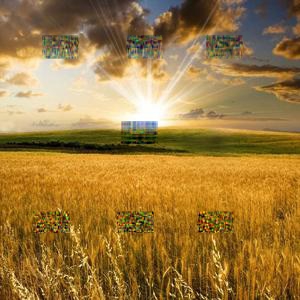}
                \vspace{0.3cm}
		\end{minipage}
	}%
	\subfigure[\parbox{4cm}{\centering recovered by ours-L-Lap \\ PSNR=31.40, SSIM=0.9885}]{
		\begin{minipage}[t]{0.33\linewidth}
			\centering
			\includegraphics[width=1.7in]{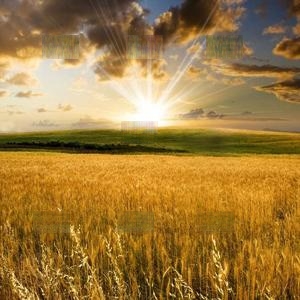}
                \vspace{0.3cm}
		\end{minipage}
	}%
	\centering
	\caption{Results of applying our method with nuclear norm and $h$-functions to an image blocked by multiple rectangular regions.}
	\label{fig:2}
\end{figure*}

\begin{figure*}[!t]
	\centering
	\subfigure[\parbox{4cm}{\centering observed image}]{
		\begin{minipage}[t]{0.33\linewidth}
			\centering
			\includegraphics[width=1.7in]{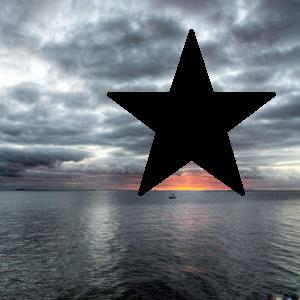}
                \vspace{0.5cm}
		\end{minipage}
	}%
	\subfigure[\parbox{4cm}{\centering recovered by ours-nuclear\\PSNR=29.27, SSIM=0.9500}]{
		\begin{minipage}[t]{0.33\linewidth}
			\centering
			\includegraphics[width=1.7in]{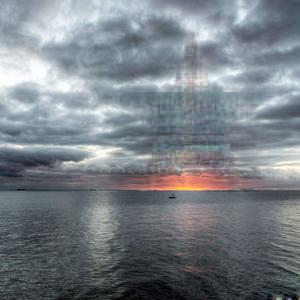}
                \vspace{0.3cm}
		\end{minipage}
	}%
	\subfigure[\parbox{4cm}{\centering recovered by ours-L-Lap \\ PSNR=29.57, SSIM=0.9503}]{
		\begin{minipage}[t]{0.33\linewidth}
			\centering
			\includegraphics[width=1.7in]{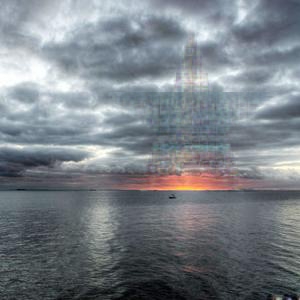}
                \vspace{0.3cm}
		\end{minipage}
	}%
	\centering
	\caption{Results of applying our method with nuclear norm and L-Lap to an image blocked by a prominent star-shaped region.}
	\label{fig:3}
\end{figure*}

To further investigate the performance of our method utilizing nuclear norm and $h$ functions, we present the experimental results in Figure \ref{fig:1}, \ref{fig:2}, and \ref{fig:3}. These figures showcase a subset of the conducted experiments where we applied both algorithms to images with various masking scenarios, enabling us to evaluate their performance across different conditions. The $h$ function we used in this experiment is Laguerre expansion-based Laplace function. Our findings consistently demonstrate that our method incorporating $h$ functions yields superior results in all image inpainting tasks. This substantiates the efficacy of $h$ functions, which approximate the pure rank, when employed in our approach.

Remarkably, even when confronted with challenging scenarios such as substantial and contiguous image regions being blocked, our method adeptly fills the missing parts based solely on the low-rank prior, thus providing plausible solutions. Notably, Figure \ref{fig:2} illustrates the block recovery problem, where our method with nuclear norm failed to achieve the desired peak signal-to-noise ratio (PSNR). However, it still exhibited a relatively high structural similarity index (SSIM), indicating that the algorithm successfully generated missing parts that closely resemble the original image's structure.

\end{document}